\PassOptionsToPackage{super,sort&compress,comma}{natbib}
\PassOptionsToPackage{hidelinks}{hyperref}
\documentclass{dukecei}

\usepackage{microtype}
\usepackage{float}
\usepackage{tikz}
\usetikzlibrary{arrows.meta,positioning,calc,fit,backgrounds,shapes.geometric,decorations.pathreplacing,patterns,shadows.blur}
\definecolor{linkblue}{HTML}{1A5DAD}
\definecolor{inkP}{HTML}{5B3A8C}
\definecolor{inkT}{HTML}{1F7A6D}
\definecolor{inkA}{HTML}{B26B00}
\definecolor{inkG}{HTML}{5A5A5A}
\definecolor{fillP}{HTML}{EEE7F6}
\definecolor{fillT}{HTML}{E2F1EE}
\definecolor{fillA}{HTML}{F6ECD9}
\definecolor{fillG}{HTML}{ECECEC}
\definecolor{fillcert}{HTML}{DDEFE0}
\definecolor{fillsusp}{HTML}{F7E4CC}
\definecolor{fillmiss}{HTML}{F3D6D6}

\theoremstyle{plain}
\newtheorem{sproposition}{Proposition}

\newtheorem{stheorem}{Theorem}

\newtheorem{scorollary}{Corollary}

\newtheorem*{thmrestate}{Proposition 3 (record-edit monotonicity; restated from the main text)}
\newtheorem*{proprestateA}{Proposition 1 (restated from the main text)}
\newtheorem*{proprestateB}{Proposition 2 (restated from the main text)}
\theoremstyle{definition}
\newtheorem{sdefinition}{Definition}

\newtheorem{assumption}{Assumption}

\theoremstyle{remark}
\newtheorem{sremark}{Remark}

\newcommand{\audita}{\textsc{Audita}}
\newcommand{\rhointerval}{[\underline{\rho},\overline{\rho}]}
\newcommand{\Gcert}{G^{\mathrm{cert}}}

\newcommand{\llmchip}[3]{%
\begin{scope}[shift={(#1,#2)}]
  \foreach \dx in {-0.18,0,0.18}{\draw[inkP,line width=0.5pt] (\dx,0.26)--(\dx,0.335);\draw[inkP,line width=0.5pt] (\dx,-0.26)--(\dx,-0.335);}
  \foreach \dy in {-0.09,0.09}{\draw[inkP,line width=0.5pt] (0.37,\dy)--(0.445,\dy);\draw[inkP,line width=0.5pt] (-0.37,\dy)--(-0.445,\dy);}
  \draw[inkP,fill=fillP,rounded corners=1.2pt,line width=0.7pt] (-0.37,-0.26) rectangle (0.37,0.26);
  \node[font=\fontsize{5.5}{6}\selectfont\bfseries,text=inkP] at (0,0.075) {LLM};
  \node[font=\fontsize{4.6}{5}\selectfont,text=inkG] at (0,-0.115) {#3};
\end{scope}}

\DeclareCaptionLabelSeparator{vbar}{~$|$~}
\renewcommand{\figurename}{Figure}

\title{\audita{}: certified auditing and causal attribution of adverse
outcomes in autonomous multi-agent systems}

\author{
    Zhixu Du$^{1,\ast}$\footnote{Correspondence E-mail: zhixu.du@duke.edu},
    Yiran Chen$^{1}$\\[1em]
    \normalsize $^{1}$Department of Electrical and Computer Engineering,
    Duke University, Durham, North Carolina, USA
}

\begin{document}

\maketitle
\thispagestyle{firstpagestyle}

\begin{abstract}
Physical automation is scaling toward fleets of embodied machines commanded by
an AI brain. Early deployments already run factories and warehouses at
production rates, and their adoption is accelerating. But
when their joint decisions cause harm, everyone involved has reason to blame
everyone else, the machine vendor, the algorithm provider, the factory
operator, the insurer, and the regulator, and no method can \emph{defensibly} divide the
responsibility between them. Existing methods read logs whose origin they
cannot verify and name a single culprit, misrepresenting outcomes that are
overdetermined, preempted, or caused by an omission. We present \audita{}, an
audit layer pairing a tamper-evident record of every inter-agent command with a
certified, graded causal-attribution engine. We prove its verdict cannot be
gamed and we establish the limit of
what an evidence-based auditor can certify. Empirically, it
reduces the standard baseline's responsibility error threefold;
on CulpaBench, our benchmark of accident-grounded structures, it recovers responsibility where
single-culprit baselines fail, and stays invariant under forgery. \audita{}
turns the question of who is to blame from an argument about logs into a
calculation over evidence.
\end{abstract}

% ===================== ARTICLE =====================
% audita-intro.tex — Introduction (NMI submission).

\section*{Introduction}
The coming operating model for physical automation is an AI brain
commanding fleets of embodied machines, and its pieces are already in place.
Language-model planners turn natural-language instructions into sequences of
robot actions\cite{ahn2022saycan,driess2023palme}, and vision-language-action
models put general-purpose control onto physical
platforms\cite{kim2024openvla,black2024pi0}. Foundation models are being built
that aim to command many bodies all at once\cite{nvidia2025groot}. On the software
side, orchestration frameworks connect specialised agents into pipelines that plan,
delegate, execute, and verify\cite{wu2023autogen,hong2023metagpt}. The interface between the two sides is now being standardised: the Model
Hardware Standard, opened in research preview in 2026, gives any agent one
shared specification for operating laboratory and manufacturing equipment
such as liquid handlers and robotic arms\cite{anthropic2026mhs}. As these
collectives enter fabrication cells, warehouses, and shared public spaces,
their joint decisions carry physical and economic consequence: a mistimed
handoff can injure a worker, and a hallucinated instruction can propagate
through three delegations before any actuator moves.

When such an adverse outcome occurs, a specific question follows, asked by the
operator deciding what to fix, the insurer pricing the loss, the regulator
demanding an account, and the vendor whose component stands accused:
\emph{which agents, through which commands, bear how much of the
responsibility, and which are demonstrably innocent?}

Today, that question has no defensible answer, because the only records that
exist were built for debugging, and those logs do not establish who wrote an
entry, in what order the entries were made, or whether any are missing. Every
one of them becomes contestable the moment liability is at stake. Regulation,
meanwhile, is already in place: the European Union's AI Act requires high-risk
systems to record events so that operations can be traced and
audited\cite{euaiact2024}, and collaborative-robot standards presume that
incidents can be reconstructed\cite{iso10218,isots15066}. The obligation is in
force. The mechanism that would discharge it does not exist.

Research attention has recently turned to the attribution half of the
problem. The Who\&When
benchmark formalised failure attribution for multi-agent systems and found the
best judge-style methods reach only $53.5\%$ agent-level and $14.2\%$
step-level accuracy\cite{zhang2025who}; a taxonomy showed that much failure
lies in inter-agent coordination rather than any single
model\cite{cemri2025mast}. Trained attributors inject and replay
faults\cite{zhang2025agentracer}, step-level counterfactual scoring converts
failures into validated repairs\cite{bonagiri2026causalflow}, and replay
engines validate an estimator against synthetic ground truth\cite{shah2026car}.
Concurrent work argues that agent systems cannot be accountable without an
auditability layer and maps its requirement dimensions\cite{nian2026auditable}
(a full survey is in the Supplementary Information). Beneath this progress sit
two structural absences.

First, \emph{an evidence layer}. Every method above reads a log whose
integrity it must assume, so its verdict is only as trustworthy as the log it
was computed from, and in any setting with liability at stake that log
\emph{will} be contested. Second, \emph{causal semantics that match how blame
behaves}: attribution targets a single culprit, but adverse outcomes are
routinely overdetermined (two commands, either sufficient), preempted (the
blamed command's effect never arrived), or caused by omission (a supervisor
who never issued the required halt), and responsibility is then a matter of
degree, not a pointer. A century of legal
doctrine\cite{wright1985,harthonore1985} and two decades of formal work on
actual causality\cite{halpernpearl2005,halpern2015,chocklerhalpern2004} treat
exactly these structures; the attribution literature for agent systems has yet
to absorb them.

We present \audita{} (Figure~\ref{fig:overview}), an audit layer that supplies
both and couples them, so that the causal analysis is computed only over
evidence that has survived verification.
Every inter-agent message is signed by its author, cites the messages it acts
on, and is sealed into an append-only, Merkle-committed
record\cite{bernstein2012,merkle1988,laurie2014}, which makes authorship,
order, and completeness cryptographically checkable. The adverse outcome
itself is not fixed in advance: the operator declares it as a predicate over
the recorded world state or over the work product, so one mechanism audits an
injury and a ruined batch alike.

Attribution then runs in two steps. A \emph{production gate} first checks whether the effect each command's author committed to
was actually realised and propagated into the state that made the outcome
true; commands whose effects never arrived are retired here, before any
counterfactual is computed. The survivors are then graded by counterfactual
dependence under the modified Halpern--Pearl definition of actual
causality\cite{halpern2015}. The grading covers every principal present in the
incident, and it returns a
responsibility degree $\rho\in[0,1]$\cite{chocklerhalpern2004} that represents
redundancy, preemption, and omission natively. Every reported cause is
re-executed under intervention\cite{kwon2023vllm}, so each claim arrives with
a receipt showing that removing that cause would have averted the outcome.

The result is reported in the four-part structure of negligence
doctrine\cite{wright1985}: which obligations were in force (\emph{duty}),
which of them the principal violated (\emph{breach}), whether the violation
made a difference to the outcome (\emph{causation}), and what the outcome cost
(\emph{harm}). Separating the four lets the verdict state what a single
culprit label cannot, for instance that a principal breached a duty but the
breach changed nothing. We intend the account as an \emph{input} to human and
legal judgment, not as a substitute for it.

Together, these components admit provable guarantees under three standard
assumptions (Methods): a re-executable stack, influence confined to the
recorded channel, and keys that are never stolen, though the insiders holding
them may be arbitrarily malicious. The guarantees state what such an
adversary cannot achieve. Causal \emph{involvement} can be manufactured
against a compliant principal; \emph{culpability} cannot: a finding requires
the accused's own certified conduct to breach a registered duty, so a
rule-following agent is never convicted (Theorem~\ref{thm:grounded}).
Shifting blame onto such an agent leaves a machine-checkable certificate
naming and grading the colluders (Theorem~\ref{thm:blameshift}), and the
protection survives attacks on the certificate itself
(Theorem~\ref{thm:hierarchy}). Destroying evidence yields flagged,
author-attributed uncertainty, never a clean acquittal
(Theorem~\ref{thm:exon}), and a completeness barrier fixes exactly what a
record-confined auditor can certify (Theorem~\ref{thm:barrier}).

We evaluate \audita{} in two settings: live
language-model multi-agent pipelines solving GSM8K\cite{cobbe2021gsm8k}, where
it cuts judge attribution error roughly threefold; and \emph{CulpaBench}, a
benchmark of physical incident structures grounded in public robot-accident
records, which we contribute as a community resource, and on which \audita{}
recovers planted responsibility exactly, refuses to blame anyone on genuine
accidents where judges always name someone, and holds its verdict invariant
under record forgery. \audita{} audits the \emph{event}, not the model: it
answers, for a concrete adverse outcome, who did what, whether it mattered, and
to what degree, with evidence that would survive a hostile audience. We argue
this layer is a precondition for autonomous collectives operating anywhere
accountability is demanded.

% audita-fig1.tex — Fig. 1: evidence layer (AI brain + fleet -> chained Merkle record
% with concrete episodes -> watching-party icons) + causal layer (investigation pipeline).
% NOTE: \llmchip is defined in audita-main.tex preamble.
\begin{figure}[!t]
\centering
\resizebox{\textwidth}{!}{%
\begin{tikzpicture}[font=\scriptsize,
  aud/.style={draw=inkA,fill=fillA,rounded corners=2pt,align=center,inner sep=2pt,font=\tiny,text width=19mm,minimum height=5.5mm,line width=0.5pt},
  watch/.style={inkA,dashed,line width=0.6pt},
  stage/.style={draw=inkP,fill=fillP,rounded corners=3pt,minimum width=27mm,minimum height=15mm,align=center,line width=0.7pt},
  cert/.style={draw=inkG,circle,fill=fillcert,minimum size=6mm,inner sep=0.5pt},
  susp/.style={draw=inkA,circle,fill=fillsusp,minimum size=6mm,inner sep=0.5pt},
  miss/.style={draw=inkG,circle,fill=fillmiss,minimum size=6mm,inner sep=0.5pt,dashed},
  big/.style={-{Stealth[length=3mm]},inkP,line width=1pt},
  blink/.style={-{Stealth[length=1.5mm]},inkP,line width=0.55pt}]

% ================= band backgrounds =================
\begin{scope}[on background layer]
  \fill[fillP!40,rounded corners=4pt] (-0.6,0.15) rectangle (17.6,4.95);
  \fill[fillT!40,rounded corners=4pt] (-0.6,-6.05) rectangle (17.6,-0.45);
\end{scope}

% ================= TOP BAND: evidence layer =================
\node[anchor=west,font=\bfseries\normalsize,text=inkP] at (-0.4,4.62) {Evidence layer: a public, verifiable command record};

% ---- machine fleet icons (left) ----
% robot arm
\begin{scope}[shift={(0.72,3.6)}]
  \draw[inkG,fill=fillG,line width=0.7pt] (-0.36,-0.4) rectangle (0.36,-0.32);
  \draw[inkG,fill=fillG,line width=0.7pt,rounded corners=0.8pt] (-0.13,-0.32) rectangle (0.13,-0.16);
  \draw[inkG,line width=1.7pt,line cap=round] (0,-0.14) -- (-0.17,0.2);
  \draw[inkG,line width=1.3pt,line cap=round] (-0.17,0.2) -- (0.17,0.36);
  \fill[inkG] (0,-0.14) circle (0.062);
  \fill[white,draw=inkG,line width=0.4pt] (-0.17,0.2) circle (0.05);
  \fill[inkG] (0.17,0.36) circle (0.04);
  \draw[inkG,line width=0.7pt] (0.17,0.36) -- (0.29,0.4);
  \draw[inkG,line width=0.6pt] (0.29,0.46) -- (0.38,0.48) (0.29,0.34) -- (0.38,0.32) (0.29,0.46) -- (0.29,0.34);
\end{scope}
% AGV cart with package
\begin{scope}[shift={(0.72,2.42)}]
  \draw[inkG,fill=fillG,rounded corners=1.6pt,line width=0.7pt] (-0.42,-0.12) rectangle (0.42,0.1);
  \fill[inkG] (-0.26,-0.14) circle (0.055) (0.26,-0.14) circle (0.055);
  \draw[inkG,fill=white,line width=0.5pt] (-0.34,0.1) rectangle (-0.22,0.19);
  \fill[inkT] (0.4,0.0) circle (0.025);
  \draw[inkA,fill=fillA,line width=0.6pt] (-0.1,0.1) rectangle (0.34,0.36);
  \draw[inkA,line width=0.45pt] (0.12,0.1) -- (0.12,0.36);
\end{scope}
% humanoid robot
\begin{scope}[shift={(0.72,1.32)}]
  \draw[inkG,fill=fillG,rounded corners=1.6pt,line width=0.7pt] (-0.11,0.22) rectangle (0.11,0.44);
  \draw[inkG,fill=white,line width=0.45pt,rounded corners=1pt] (-0.07,0.3) rectangle (0.07,0.37);
  \draw[inkG,fill=fillG,line width=0.7pt,rounded corners=1.6pt] (-0.16,-0.1) -- (-0.13,0.18) -- (0.13,0.18) -- (0.16,-0.1) -- cycle;
  \draw[inkG,line width=1.0pt,line cap=round] (-0.15,0.13) -- (-0.25,-0.08);
  \draw[inkG,line width=1.0pt,line cap=round] (0.15,0.13) -- (0.25,-0.08);
  \draw[inkG,line width=1.2pt,line cap=round] (-0.07,-0.1) -- (-0.07,-0.38);
  \draw[inkG,line width=1.2pt,line cap=round] (0.07,-0.1) -- (0.07,-0.38);
\end{scope}
\node[font=\tiny,text=inkG] at (0.72,0.6) {machine fleet};

% ---- fleet <-> brain arrows ----
\draw[-{Stealth[length=1.8mm]},inkT,line width=0.8pt] (2.38,2.86) -- (1.32,2.86);
\node[font=\tiny,text=inkT] at (1.8,3.14) {commands};
\draw[-{Stealth[length=1.8mm]},inkG,line width=0.8pt] (1.32,2.18) -- (2.38,2.18);
\node[font=\tiny,text=inkG] at (1.8,1.9) {sensor data};

% ---- AI brain: one controller, many LLMs (hub and spoke) ----
\node[ellipse,draw=inkP,fill=white,line width=1pt,minimum width=39mm,minimum height=30mm] (brain) at (4.3,2.55) {};
\llmchip{4.3}{3.42}{planner}
\llmchip{3.42}{2.5}{solver}
\llmchip{5.18}{2.5}{solver}
\llmchip{4.3}{1.6}{aggregator}
\draw[blink] (4.08,3.12) -- (3.55,2.85);
\draw[blink] (4.52,3.12) -- (5.05,2.85);
\draw[blink] (3.55,2.16) -- (4.08,1.9);
\draw[blink] (5.05,2.16) -- (4.52,1.9);
\begin{scope}[shift={(4.3,2.52)},scale=0.62]
  \draw[inkP,fill=fillP!75,line width=0.6pt]
    (0,0.3) .. controls (-0.2,0.35) and (-0.33,0.22) .. (-0.32,0.07)
    .. controls (-0.42,-0.02) and (-0.34,-0.2) .. (-0.2,-0.21)
    .. controls (-0.16,-0.3) and (-0.02,-0.32) .. (0,-0.24)
    .. controls (0.02,-0.32) and (0.16,-0.3) .. (0.2,-0.21)
    .. controls (0.34,-0.2) and (0.42,-0.02) .. (0.32,0.07)
    .. controls (0.33,0.22) and (0.2,0.35) .. (0,0.3) -- cycle;
  \draw[inkP,line width=0.45pt] (0,0.28) -- (0,-0.22);
  \draw[inkP,line width=0.4pt] (-0.19,0.1) .. controls (-0.09,0.15) .. (-0.07,0.02);
  \draw[inkP,line width=0.4pt] (0.19,0.08) .. controls (0.1,0.13) .. (0.08,0.0);
\end{scope}
\node[font=\tiny\bfseries,text=inkP] at (4.05,0.6) {AI brain: one controller, many LLMs};

% ---- signed commands into the ledger ----
\draw[-{Stealth[length=2.2mm]},inkT,line width=1.1pt] (6.35,2.5) -- (7.13,2.5);
\node[font=\tiny,text=inkT,align=center] at (6.68,3.16) {signed\\ commands};

% ---- blockchain record: chained epochs; newest block shows a concrete episode ----
% block t (newest): concrete commands from a warehouse episode
\draw[inkT,fill=white,line width=0.7pt] (7.2,1.65) rectangle (9.8,3.35);
\draw[inkT,fill=fillT,line width=0.7pt] (7.2,3.02) rectangle (9.8,3.35);
\node[font=\fontsize{4.8}{5}\selectfont,text=inkT] at (8.5,3.185) {prev \# $\cdot$ root $\sigma$};
\node[font=\fontsize{3.85}{4.3}\selectfont,text=inkG,anchor=west] at (7.31,2.76) {m\textsubscript{45} planner\,\scalebox{0.75}{$\to$}\,agv\textsubscript{2} move(aisle 2)};
\node[font=\fontsize{3.85}{4.3}\selectfont,text=inkG,anchor=west] at (7.31,2.46) {m\textsubscript{46} agv\textsubscript{2}\,\scalebox{0.75}{$\to$}\,arm\textsubscript{1} pick(crate 7)};
\node[font=\fontsize{3.85}{4.3}\selectfont,text=inkG,anchor=west] at (7.31,2.16) {m\textsubscript{47} checker\,\scalebox{0.75}{$\to$}\,agv\textsubscript{2} halt};
\draw[inkG!55,line width=0.5pt] (7.5,1.9) -- (9.5,1.9);
\node[font=\tiny,text=inkG] at (8.85,1.26) {epoch $t$};
% blocks t-1, t-2 (sealed history: earlier episodes of the same shift)
\foreach \bx/\ep/\ra/\rb/\rc in {
  11.1/{$t{-}1$}/{m\textsubscript{42} planner\,\scalebox{0.75}{$\to$}\,agv\textsubscript{1} move(dock)}/{m\textsubscript{43} agv\textsubscript{1}\,\scalebox{0.75}{$\to$}\,arm\textsubscript{2} pick(bin 4)}/{m\textsubscript{44} verifier\,\scalebox{0.75}{$\to$}\,planner report ok},
  13.55/{$t{-}2$}/{m\textsubscript{39} planner\,\scalebox{0.75}{$\to$}\,arm\textsubscript{2} align}/{m\textsubscript{40} arm\textsubscript{2}\,\scalebox{0.75}{$\to$}\,planner done}/{m\textsubscript{41} checker\,\scalebox{0.75}{$\to$}\,arm\textsubscript{2} resume}}{
  \draw[inkT,fill=white,line width=0.7pt] (\bx-1.15,1.65) rectangle (\bx+1.15,3.35);
  \draw[inkT,fill=fillT,line width=0.7pt] (\bx-1.15,3.02) rectangle (\bx+1.15,3.35);
  \node[font=\fontsize{4.8}{5}\selectfont,text=inkT] at (\bx,3.185) {prev \# $\cdot$ root $\sigma$};
  \node[font=\fontsize{3.85}{4.3}\selectfont,text=inkG,anchor=west] at (\bx-1.09,2.76) {\ra};
  \node[font=\fontsize{3.85}{4.3}\selectfont,text=inkG,anchor=west] at (\bx-1.09,2.46) {\rb};
  \node[font=\fontsize{3.85}{4.3}\selectfont,text=inkG,anchor=west] at (\bx-1.09,2.16) {\rc};
  \draw[inkG!55,line width=0.5pt] (\bx-0.9,1.9) -- (\bx+0.9,1.9);
  \node[font=\tiny,text=inkG] at (\bx,1.26) {\ep};
}
\draw[-{Stealth[length=1.6mm]},inkT,line width=0.7pt] (9.81,3.185) -- (9.94,3.185);
\draw[-{Stealth[length=1.6mm]},inkT,line width=0.7pt] (12.26,3.185) -- (12.39,3.185);
\node[font=\tiny,text=inkT] at (10.95,3.62) {each header cites the previous hash};
\node[font=\tiny,text=inkG] at (10.4,0.74) {append-only, Merkle-sealed command record};
\node[font=\fontsize{5}{5.5}\selectfont,text=inkG] at (10.4,0.44) {every message $m=(a,\pi,C,\varphi,\sigma)$ signed by its author (Definition~\ref{def:graph})};

% ---- the watching parties: icons on a half-circle around the record ----
% machine vendor: factory
\begin{scope}[shift={(14.55,4.35)},scale=1.9]
  \draw[inkA,line width=0.55pt,fill=fillA!60] (-0.1,-0.065) -- (-0.1,0.065) -- (-0.04,0.015) -- (-0.04,0.065) -- (0.02,0.015) -- (0.02,0.065) -- (0.1,0.065) -- (0.1,-0.065) -- cycle;
  \draw[inkA,line width=0.55pt] (0.05,0.065) -- (0.05,0.135) (0.09,0.065) -- (0.09,0.135) (0.05,0.135) -- (0.09,0.135);
\end{scope}
\node[font=\tiny,text=inkA] at (14.55,3.95) {machine vendor};
% algorithm provider: chip
\begin{scope}[shift={(16.1,3.5)},scale=1.9]
  \draw[inkA,line width=0.55pt,fill=fillA!60] (-0.075,-0.075) rectangle (0.075,0.075);
  \foreach \p in {-0.045,0,0.045}{\draw[inkA,line width=0.45pt] (\p,0.075)--(\p,0.115);\draw[inkA,line width=0.45pt] (\p,-0.075)--(\p,-0.115);}
\end{scope}
\node[font=\tiny,text=inkA] at (16.1,3.1) {algorithm provider};
% factory operator: hard-hat worker
\begin{scope}[shift={(16.85,2.56)},scale=1.9]
  \fill[inkA] (0,0.02) circle (0.05);
  \draw[inkA,line width=0.8pt] (-0.068,0.037) arc (155:25:0.075);
  \draw[inkA,line width=0.7pt] (-0.08,-0.13) arc (180:0:0.08);
\end{scope}
\node[font=\tiny,text=inkA] at (16.85,2.12) {factory operator};
% insurer: shield
\begin{scope}[shift={(16.1,1.5)},scale=1.9]
  \draw[inkA,line width=0.55pt,fill=fillA!60] (-0.09,0.09) -- (0.09,0.09) -- (0.09,0.01) .. controls (0.09,-0.07) and (0.04,-0.11) .. (0,-0.135) .. controls (-0.04,-0.11) and (-0.09,-0.07) .. (-0.09,0.01) -- cycle;
  \draw[inkA,line width=0.45pt] (0,0.05) -- (0,-0.06);
\end{scope}
\node[font=\tiny,text=inkA] at (16.1,1.1) {insurer};
% regulator: balance scale
\begin{scope}[shift={(14.55,1.05)},scale=1.9]
  \draw[inkA,line width=0.55pt] (0,-0.11)--(0,0.09) (-0.055,-0.11)--(0.055,-0.11) (-0.1,0.05)--(0.1,0.05);
  \draw[inkA,line width=0.45pt] (-0.1,0.05)--(-0.145,-0.02) (-0.1,0.05)--(-0.055,-0.02);
  \draw[inkA,line width=0.45pt] (-0.145,-0.02) arc (180:360:0.045);
  \draw[inkA,line width=0.45pt] (0.1,0.05)--(0.055,-0.02) (0.1,0.05)--(0.145,-0.02);
  \draw[inkA,line width=0.45pt] (0.055,-0.02) arc (180:360:0.045);
\end{scope}
\node[font=\tiny,text=inkA] at (14.55,0.65) {regulator};
% watch lines to the record
\draw[watch] (14.3,3.88) -- (14.05,3.44);
\draw[watch] (15.5,3.0) -- (14.78,2.85);
\draw[watch] (16.5,2.5) -- (14.78,2.5);
\draw[watch] (15.75,1.55) -- (14.78,2.1);
\draw[watch] (14.55,1.35) -- (14.52,1.62);
\node[font=\tiny\itshape,text=inkA,align=center] at (16.45,0.52) {any party can verify;\\ none can rewrite};

% ================= BOTTOM BAND: causal layer =================
\node[anchor=east,font=\bfseries\normalsize,text=inkP] at (17.4,-0.78) {Causal layer: one investigation, end to end};

\node[draw=inkA,fill=fillA,rounded corners=3pt,minimum width=18mm,minimum height=15mm,align=center] (inc)
  at (0.5,-3.1) {\textbf{Adverse}\\\textbf{outcome}\\\footnotesize declared\\predicate\\$Y{=}1$};
\node[stage,right=6mm of inc,minimum width=30mm,minimum height=22mm] (slice) {};
\node[font=\scriptsize\bfseries,text=inkP] at ([yshift=5.5mm]slice.north) {1. Slice + integrity partition};
\begin{scope}[shift={(slice.center)},scale=0.9]
  \node[cert] (n1) at (-1.1,0.4) {};
  \node[cert] (n2) at (-0.4,0.7) {};
  \node[susp] (n3) at (-0.4,-0.3) {};
  \node[cert] (n4) at (0.4,0.4) {};
  \node[miss] (n5) at (0.4,-0.6) {};
  \node[cert] (n6) at (1.1,0.1) {};
  \draw[-{Stealth[length=1mm]},inkG] (n1)--(n2); \draw[-{Stealth[length=1mm]},inkG] (n2)--(n4);
  \draw[-{Stealth[length=1mm]},inkG] (n3)--(n4); \draw[-{Stealth[length=1mm]},inkG] (n4)--(n6);
  \draw[-{Stealth[length=1mm]},inkG,dashed] (n5)--(n6);
\end{scope}
\node[font=\tiny,text=inkG] at ([yshift=-6mm]slice.south)
  {\tikz\node[cert,scale=0.6]{};\,certified \tikz\node[susp,scale=0.6]{};\,suspect \tikz\node[miss,scale=0.6]{};\,missing};
\node[stage,right=6mm of slice] (gate) {2. Production\\gate\\\footnotesize keep intact\\certified paths};
\node[stage,right=6mm of gate,minimum width=29mm] (eng) {3. Two-layer engine\\\footnotesize propose $(X^\ast,W^\ast)$\\certify by replay};
\node[draw=inkP,fill=white,rounded corners=3pt,minimum width=49mm,minimum height=30mm,line width=1pt,align=left,right=6mm of eng,yshift=0mm] (verd) {};
\node[font=\scriptsize\bfseries,text=inkP] at ([yshift=12mm]verd.center) {4. Graded verdict (per principal)};
\begin{scope}[shift={(verd.center)},yshift=-1mm]
  \node[font=\tiny,anchor=east] at (-1.0,0.7) {planner (m\textsubscript{45})};
  \draw[inkG] (-0.9,0.6) rectangle (1.3,0.85); \fill[inkP!55] (-0.9,0.6) rectangle (1.15,0.85);
  \draw[inkA,line width=0.8pt] (1.0,0.725)--(1.3,0.725);
  \node[font=\tiny,anchor=east] at (-1.0,0.25) {agv\textsubscript{2} (m\textsubscript{46})};
  \draw[inkG] (-0.9,0.15) rectangle (1.3,0.4); \fill[inkP!55] (-0.9,0.15) rectangle (1.15,0.4);
  \draw[inkA,line width=0.8pt] (1.0,0.275)--(1.3,0.275);
  \node[font=\tiny,text=inkG,anchor=west] at (1.36,0.275) {conduit};
  \node[font=\tiny,anchor=east] at (-1.0,-0.2) {verifier};
  \draw[inkG] (-0.9,-0.3) rectangle (1.3,-0.05); \fill[inkT!55] (-0.9,-0.3) rectangle (0.15,-0.05);
  \node[font=\tiny,text=inkT,anchor=west] at (1.36,-0.175) {duty!};
  \node[font=\tiny,anchor=east] at (-1.0,-0.65) {checker (m\textsubscript{47})};
  \draw[inkG] (-0.9,-0.75) rectangle (1.3,-0.5);
  \node[font=\tiny,text=inkG,anchor=west] at (-0.82,-0.625) {$\rho{=}0$, halt issued};
  \node[font=\tiny,text=inkG] at (0.5,-1.05) {bars: $\rho$ + interval; teal: omission};
\end{scope}
\draw[big] (inc)--(slice);
\draw[big] (slice)--(gate);
\draw[big] (gate)--(eng);
\draw[big] (eng)--(verd);
\draw[-{Stealth[length=2mm]},inkA,dashed] (verd.south) .. controls +(0,-8mm) and +(0,-8mm) .. (slice.south)
  node[midway,below,font=\scriptsize\itshape,text=inkA]{every bar links back to its certified evidence sub-graph};

% ---- coupling ----
\draw[inkP,dashed,line width=0.8pt] (7.35,1.5) -- (14.45,1.5);
\draw[-{Stealth[length=2.2mm]},inkP,dashed,line width=0.8pt]
  (7.35,1.5) .. controls +(-0.35,-1.75) and +(0,1.45) .. ([yshift=8.5mm]slice.north);
\node[font=\tiny,text=inkP,anchor=west] at (8.4,-0.14) {an incident triggers a backward slice across all sealed epochs};

\end{tikzpicture}}
\caption{\textbf{\audita{}'s two layers.}
\emph{Top, evidence layer.} An AI brain, one controller running many language
models (a planner, parallel solvers, an aggregator), commands a machine fleet
and receives its sensor data. Every inter-agent command is a signed message
$m=(a,\pi,C,\varphi,\sigma)$, author, payload, citations, committed effect,
signature, appended with its delivery receipt to a public record in the
blockchain sense (Definition~\ref{def:graph}). An \emph{epoch} is a fixed
window of this message stream: at its close the recorder signs the Merkle root
over the window's messages and chains it to the previous block's hash, so the
record is append-only and tamper-evident. The blocks show illustrative
episodes of one shift (agv, an automated guided vehicle; arm, a manipulator
arm): in the newest block the planner routes an AGV, the AGV hands a crate
to an arm, and the checker issues a halt; the sealed blocks hold earlier,
already-committed windows. The machine vendor, the algorithm provider, the
factory operator, the insurer, and the regulator all hold the same root: any
party can verify the record, and none can rewrite it.
\emph{Bottom, causal layer.} A declared adverse outcome $Y{=}1$ triggers a
backward slice of the entire sealed record, reaching across epochs to every
certified message with a citation path into the outcome; each message in the
slice is coloured by the verification partition (certified, suspect,
missing). The production gate
keeps only messages with an intact certified path into the outcome. Surviving
candidates enter the two-layer engine (Extended Data
Fig.~\ref{fig:engines}), which proposes minimal cause--witness sets and
certifies them by counterfactual replay. The output is a graded verdict per
principal: duty and breach findings, a responsibility interval $\rhointerval$
with verification receipts, and conduit or exoneration flags, with every bar
linked to the certified sub-graph that justifies it. Suspect and missing
evidence widens intervals but, by Proposition~\ref{prop:recordedit}, can never
raise the involvement bound of a fully certified principal; and by
Theorem~\ref{thm:grounded}, culpability additionally requires that principal's
own certified breach.}
\label{fig:overview}
\end{figure}
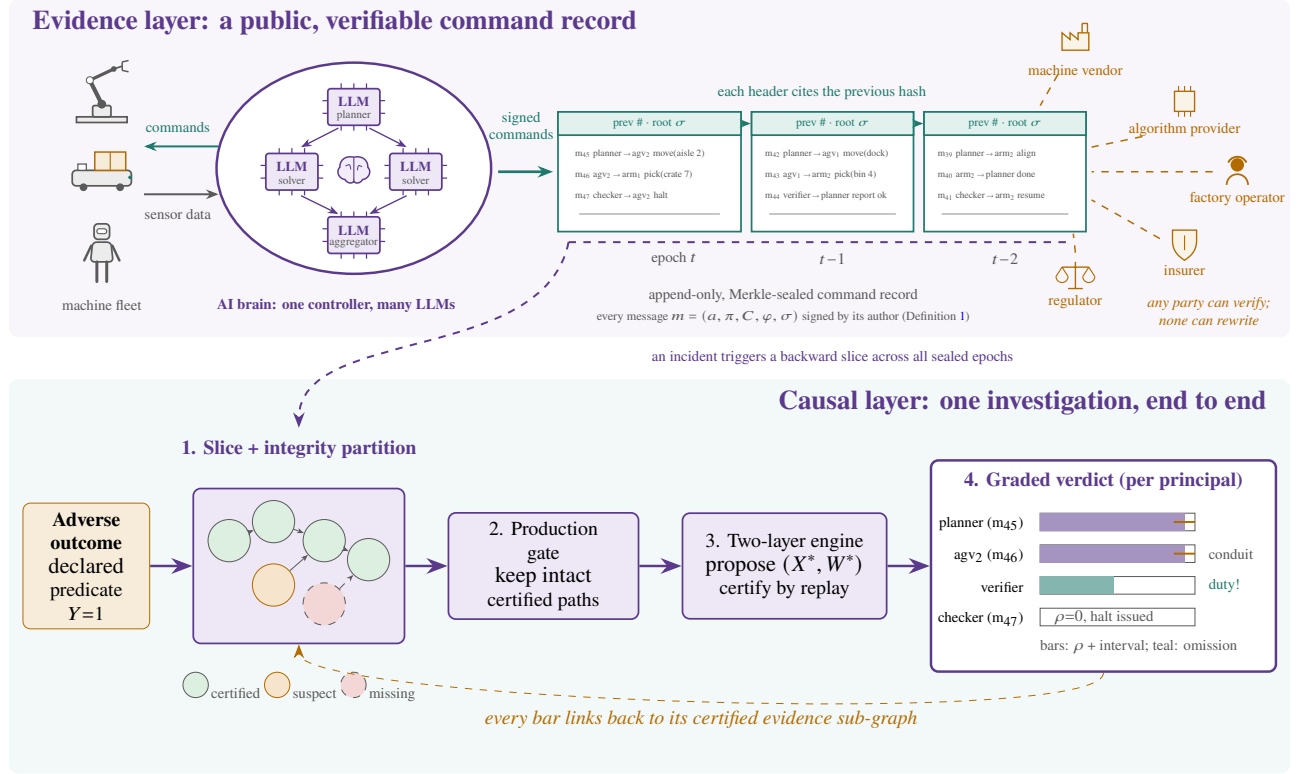

% audita-methods.tex

\section*{Methods}

\subsection*{Setting, threat model, and assumptions}
We consider a collective of \emph{principals}
$P=\{p_1,\dots,p_n\}$: language-model planners, tool agents, robot
controllers, and human supervisors, interacting through an asynchronous
message bus and, for embodied members, through actuators in a shared
environment. Principals hold signing keys; a designated (replicated)
\emph{recorder} appends messages to the audit record. The adversary may
control any subset of principals and the network in the Dolev--Yao
sense\cite{dolev1983}: it can read, delay, reorder, replay, drop, and inject
traffic, may equivocate (present different histories to different
parties\cite{lamport1982}), and may attack the record itself by forging
entries or deleting them after the fact. An insider who signs malicious
content with a \emph{valid} key is not an attack on the record: their
messages certify normally and are attributed to them, which is the desired
outcome, and their misconduct is surfaced by the verdict's breach analysis
rather than by integrity machinery.

The guarantees below rest on three explicit assumptions, each standard in
its home literature.
\begin{description}
\item[A1 (Replayable stack).] The audited deployment runs on components
that can be re-executed exactly: pinned model weights, deterministic
serving, and seeded environments. Under A1 every replay receipt in this
paper is exact and bitwise checkable. On hosted or non-deterministic
components the same checks degrade to statistical corroboration with
reported action-match.
\item[A2 (Recorded channel).] Inter-principal influence traverses the
recorded bus. Under A2 exonerations are complete. Influence outside the bus
(shared environment state, out-of-band instruction) is invisible to any
record-level auditor, and a dedicated experiment in Results measures the
attribution mass at stake when A2 is violated deliberately (Costs and limits in Results).
\item[A3 (Sound key custody).] Signing keys are never stolen, so certified
authorship equals conduct. A3 restricts nothing about behaviour: key
holders, including insiders, may be arbitrarily malicious. Key management
itself is inherited from the transparency-log tradition\cite{laurie2014}
rather than re-solved here.
\end{description}
The asymmetry in A3 is deliberate and load-bearing: the adversary keeps
every capability that matters for accountability, including acting
maliciously during the incident under a valid identity. The guarantees are
statements about what such an adversary still cannot achieve.

\paragraph{The accountable command graph.} The unit of evidence is the
message.

\begin{definition}[Accountable message and command graph]
\label{def:graph}
A message is a tuple $m=(a,\pi,C,\varphi,\sigma)$: author $a\in P$, payload
$\pi$, citation set $C$ (hashes of the messages $m$ acts upon), effect
predicate $\varphi$ (the world- or work-state change the author commits to,
enabling later verification), and signature $\sigma$ over
$(a,\pi,C,\varphi)$ with $a$'s key\cite{bernstein2012}. The record $R$ is the
multiset of received messages together with delivery receipts; epochs of $R$
are sealed by a Merkle commitment\cite{merkle1988} published append-only in
the style of transparency logs\cite{laurie2014}. The \emph{command graph}
$G(R)$ has messages as nodes and citation edges $m'\!\to\!m$ for each
$c\in C(m)$ resolving to $m'$.
\end{definition}

Citation is an admission: by citing $m'$, the author of $m$ attests that $m$ was
issued \emph{because of} $m'$, making the edge an authored causal claim rather
than an inferred correlation. Verification partitions the record into
$\Gcert$ (signature valid, citations resolve, receipts consistent, Merkle
path intact), a \emph{suspect} set (any check fails), and a \emph{missing}
set (cited or receipted but absent). All attribution downstream operates on
$\Gcert$; the suspect and missing sets enter only as explicit uncertainty
(Proposition~\ref{prop:recordedit}). Extended Data Figure~\ref{fig:anatomy} shows the anatomy and
the partition.

\paragraph{Adverse outcomes and incident slices.} \audita{} is
outcome-agnostic: the trigger is a declared predicate, not a category of
accident.

\begin{definition}[Adverse outcome and incident slice]
\label{def:outcome}
An \emph{adverse outcome} is a predicate $Y$ over the recorded world- and
work-state that evaluates true at some time $t^{\ast}$: for example, a safety
predicate (minimum human--robot separation below the collaborative-operation
threshold\cite{isots15066}) or a quality predicate (a work product failing its
declared acceptance test). The \emph{incident slice} $S(Y)$ is the causal
past of $Y$ in $\Gcert$: all certified messages with a citation path into the
actuation events referenced by $Y$, closed under the duties active in the
window (below).
\end{definition}

Because $Y$ is arbitrary, the same machinery audits an injury, a near-miss,
and a ruined batch; the facility case study instantiates both a safety and a
quality $Y$ (Results).

\subsection*{Certified causal attribution: gate, grade, certify}
We build a structural causal
model\cite{pearl2009,halpernpearl2005} $M_{S}$ whose binary variables are the
certified messages (present/absent), duty-indexed \emph{absence} variables for
required-but-unissued interventions (an omission such as a never-sent halt is
a first-class variable, not a gap), and the mechanism equations induced by the
citation structure and the recorded actuation semantics. Attribution is a
two-part test.

\emph{Gate (factual production).} A candidate set of messages $X$ enters
consideration only if each member has an intact production path to $Y$ in
$\Gcert$: its committed effect $\varphi$ was realised and propagated into the state that made $Y$ true. The gate retires
\emph{preempted} candidates (commands whose effect never reached the outcome)
before any counterfactual is computed, in the spirit of the law's insistence
on causation-in-fact\cite{wright1985,harthonore1985}.

\emph{Grade (counterfactual dependence).} Surviving candidates are graded
under the modified Halpern--Pearl definition of actual
causality\cite{halpern2015,halpern2016book}: $X^{\ast}$ is an actual cause of
$Y$ if there is a witness set $W^{\ast}$, frozen at its actual values, such
that setting $X^{\ast}$ to its counterfactual values falsifies $Y$, with
$X^{\ast}$ minimal. Following the responsibility calculus of Chockler and
Halpern\cite{chocklerhalpern2004}, each minimal cause--witness pair receives
the grade
\begin{equation}
\label{eq:rho}
\rho(X^{\ast},W^{\ast}) \;=\; \frac{1}{\lvert X^{\ast}\rvert + \lvert
W^{\ast}\rvert},
\end{equation}
recovering the classical degree of responsibility $1/(k{+}1)$ through the
size correspondence between the modified and original
definitions\cite{halpern2015}. A principal's responsibility is the strongest grade any of
their messages participates in:
\begin{equation}
\label{eq:rhop}
\rho_p \;=\; \max\Bigl\{\, \rho(X^{\ast},W^{\ast}) \;:\; X^{\ast}\ \text{a
minimal certified cause of } Y \text{ containing a message authored by } p
\,\Bigr\},
\end{equation}
and $\rho_p=0$ with an explicit \emph{exoneration note} if no such cause
exists. These degrees are not a normalised allocation and need not sum to one
across principals: each is a Chockler--Halpern degree of responsibility for
the outcome, not a share of it. Redundancy (two independently sufficient commands) yields
$\rho=\tfrac12$ each rather than an arbitrary single culprit; omission is
graded through its absence variable; a principal whose messages merely relay
upstream content faithfully is additionally flagged as a \emph{conduit}, so
that the verdict can distinguish originating from transmitting
responsibility. Because Eq.~\eqref{eq:rhop} takes a maximum over causes, the
grade is invariant to how a principal chunks its own output across messages;
a canonical contraction of same-author conjunctive groups makes this
invariance explicit and is proven in Supplementary Note~6.

\emph{Certify (counterfactual replay).} Grades proposed on the structural
model are not reported until re-executed. For each reported cause, the replay
engine intervenes on $X^{\ast}$ (holding $W^{\ast}$ at actuals), re-runs the
slice, and checks that the certified risk of $Y$ falls by at least a declared
margin $\delta$. On a \emph{pinned} stack (open-weight models, seeded
simulation, versioned tools, and the \emph{serial replay protocol}, under
which replays execute one at a time on an otherwise idle
server\cite{kwon2023vllm}) replay is exact, in the lineage of record-and-replay
systems\cite{ocallahan2017}. Serial execution is part of the pinned
definition, not an optimisation: concurrent batching perturbs floating-point
reduction order and voids bitwise exactness even in batch-invariant serving
modes, so a pinned receipt certifies the protocol along with the seeds. On a \emph{hosted} stack (closed models,
non-deterministic serving) the same check is \emph{statistically
corroborated}: repeated replays yield a confidence interval on the risk
reduction, reported together with an \emph{action-match} score measuring how
faithfully the replayed trajectory tracks the recorded one. Every reported
cause therefore carries a verification receipt, exact or statistical, and a
cause that fails its replay is not reported (Proposition~\ref{prop:sound}).
Deciding actual causation is NP-hard in general\cite{eiter2002}; \audita{}
computes exhaustively over incident slices and we \emph{measure} the
practical boundary (slice width at which exact computation exceeds an
interactive budget) rather than assert scalability. Extended Data Fig.~\ref{fig:engines} shows the propose--certify
pipeline.

\begin{proposition}[Soundness by certification]
\label{prop:sound}
Every cause reported by \audita{} is replay-verified: intervening on
$X^{\ast}$ with $W^{\ast}$ frozen reduces the certified risk of $Y$ by at
least $\delta$: exactly on a pinned stack, and at confidence $1-\alpha$ with
reported action-match on a hosted stack. Consequently certified false
positives are excluded on pinned stacks and bounded by $\alpha$ on hosted
ones; the engine's failure mode is abstention (a missed cause), not a false
accusation. Proof in Supplementary Note~2.
\end{proposition}

\begin{proposition}[Preemption retirement]
\label{prop:preempt}
A message with no intact production path to $Y$ in $\Gcert$ receives
$\rho=0$ and an exoneration note, regardless of its content. Proof in
Supplementary Note~2.
\end{proposition}

\begin{remark}
On monotone incident models, which include all scenario families in this
paper, the gate-and-grade test coincides with the minimal
sufficient-set analysis of the NESS (necessary element of a sufficient set)
test from tort doctrine\cite{wright1985}; the correspondence is stated and
proven in Supplementary Note~2. This is deliberate: the quantity \audita{}
computes is one a legal audience already recognises.
\end{remark}

\subsection*{Guarantees under record attack and conduct attack}
The composition of verification and attribution is designed to survive an
adversary with liability at stake. \audita{} reports each $\rho_p$ as an
interval $\rhointerval_p$: the range of Eq.~\eqref{eq:rhop} over all
\emph{admissible completions} of the certified record: assignments of
present/absent to the missing set consistent with the verified receipts and
commitments, with suspect items excluded from grading (treated as absent)
and reported separately. The lower end $\underline{\rho}_p$ is the
\emph{involvement bound}: the causal involvement that the certified evidence
alone forces. Involvement is deliberately not culpability: the verdict layer
(next subsection) finds a principal culpable only on breach \emph{and}
causation together, and the guarantees below attach to exactly that
composition. Formal statements, the attack constructions, and full proofs
are in Supplementary Note~3; the record-layer proposition is proven in
Supplementary Note~2.

\begin{proposition}[Record-edit monotonicity]
\label{prop:recordedit}
Let $p$ be a principal all of whose messages verify (fully certified), and
let an adversary modify the record by (i) injecting messages that fail
verification (forged signatures, unresolvable citations, broken Merkle
paths) and/or (ii) deleting messages (deletions against sealed epochs;
Supplementary Note~2). Then the involvement bound $\underline{\rho}_p$ does
not increase: forgeries leave $\Gcert$, and hence the entire reported
verdict, unchanged; deletions can only widen the reported interval
$\rhointerval_p$, never raise its lower end.
\end{proposition}

Proposition~\ref{prop:recordedit} is a \emph{mechanism check}, and we label
it as one: non-verifying items never enter $\Gcert$, and a deletion enlarges
the completion set over which a minimum is taken. It is silent about the
strongest adversary the threat model admits, a coalition holding
\emph{valid} keys and acting during the incident. Such a coalition needs no
forgery. An insider running the live policy ``if the victim's message is
present, emit the harmful command, citing it; otherwise behave'' turns an
innocent principal into a true but-for cause of the outcome: every item
certifies, the production gate passes through the attacker's own citation,
and honest replay certifies $\rho_v = 1$ (Supplementary Note~3, Attack~I). A
second construction (Attack~II) achieves the same raise with a deviation that
lies on no path from the victim to the outcome and is not itself a cause of
the outcome: it rigs an otherwise-inevitable harm so that only the victim
could have prevented it. Causal involvement is therefore manufacturable by
validly-signed conduct, and no record machinery can prevent it. The following
results state what remains impossible.

\begin{definition}[Content-determined duty]
\label{def:cdd}
A duty for principal $p$ is a computable predicate $\Phi_p$ over exactly
(i)~the certified messages authored by $p$ (and $p$'s duty-registered
absences) and (ii)~their certified input frontiers. $\Phi_p$ takes no
ground-truth label and no third-party assertion \emph{about} $p$: another
principal's payload bears on $p$'s verdict only as an input that $p$ acted
on. The duties in this paper are of this form; for example, an aggregator
must output the declared normative aggregation of its certified inputs, and
a solver's emitted answer must match the answer derived in its own certified
working.
\end{definition}

\begin{theorem}[Culpability groundedness]
\label{thm:grounded}
Under any content-determined duty roster, for every valid-key coalition,
every principal $p$, and every coalition-induced run with any record edits of
Proposition~\ref{prop:recordedit}'s class applied: if the verdict finds $p$
culpable (breach with $\underline{\rho}_p>0$), then the breach finding is
witnessed by $p$'s own certified conduct violating $\Phi_p$. Equivalently, a
principal whose own certified conduct was duty-compliant is never found
culpable; the strongest verdict against such a principal is \emph{causal
involvement with exoneration for breach}.
\end{theorem}

\begin{theorem}[Blame-shift accountability]
\label{thm:blameshift}
Declare a reference standard: a registered duty-compliant policy per
principal. If a coalition's conduct deviations from that standard raise a
compliant, fully-certified principal's involvement bound, then the deviations
are an actual cause of that principal's \emph{pivotality}, and a minimal
\emph{blame-shift certificate} exists: a set of deviations whose restoration
to the declared standard destroys the raise, machine-checkable by replay,
grading each framer by the same responsibility calculus
($\rho^{\mathrm{frame}}_a$). No stronger localisation holds: by Attack~II the
deviation need lie on no path to the outcome and need not cause the outcome.
\end{theorem}

\begin{theorem}[No silent exoneration]
\label{thm:exon}
Let $q$ be a certified culprit, $\underline{\rho}_q(R) = r > 0$ on the sealed
record. Any record edit that lowers $\underline{\rho}_q$ leaves a nonempty
missing set meeting every family of certified items that forced level $r$,
each missing item reported with its author and citation frontier; and the
upper end $\overline{\rho}_q$ never falls below $r$. The same holds for the
breach prong: erasing or corrupting the certified witnesses of a breach
finding moves them, author-attributed, to the missing or suspect set rather
than yielding a clean non-breach. Destruction converts forced responsibility
into flagged, pre-attributed uncertainty, never into certified innocence,
on either prong.
\end{theorem}

\begin{theorem}[Completeness barrier]
\label{thm:barrier}
Call a content-determined breach standard \emph{sound} if it passes all
duty-compliant conduct, and \emph{complete} if it fires on all conduct whose
committed output was wrong with respect to ground truth and caused the
adverse outcome. Wherever duty-compliant conduct is fallible (honest errors
exist), no content-determined standard is both sound and complete; and any
standard that closes the gap decides output correctness for the audited task
itself. Innocence is certifiable from the record alone;
outcome-guilt-completeness is not.
\end{theorem}

Theorem~\ref{thm:grounded} formalises the presumption of innocence for
certified records: culpability moves only on the accused's own certified
conduct. Theorem~\ref{thm:blameshift} makes manufactured involvement itself
attributable: the framing is graded by the calculus it abused. Three
adjacent lines sharpen what this certificate is and is not.
Security-protocol accountability treats deviations from a specification as
actual causes of a \emph{violation}\cite{datta2015program,
kunnemann2019accountability}; here the deviations are causes of another
principal's \emph{pivotality}, one level up. Non-frameability notions
guarantee that honest parties are never \emph{falsely}
blamed\cite{kuesters2010accountability,haeberlen2007peerreview}; the
blame-shift victim is \emph{genuinely} pivotal, a case those definitions
cannot express, and the certificate shows the true pivotality was
manufactured. And higher-order responsibility in strategic games asks who
could have prevented a responsibility \emph{gap}\cite{jiang2026higherorder};
the certificate instead grades who \emph{created} a responsibility
\emph{surplus} on a compliant victim. To our knowledge it is the first
causal certificate of manufactured true pivotality with graded framer
responsibility, and we claim exactly that.
Theorem~\ref{thm:exon} is the dual guarantee against manufactured innocence.

\begin{theorem}[Grounded framing at every order]
\label{thm:hierarchy}
Fix the certified record and the registered content-determined duty roster,
and let the order-one certificate be as in Theorem~\ref{thm:blameshift}. For
$k \ge 1$ define the order-$k$ certificate by taking as endogenous variables
deviation selectors over conduct on the record, each owned by the signer of
its realised conduct, and as outcome event any pivotality predicate definable
from the order-$(k-1)$ analysis. Then (i) \emph{completeness}: whenever the
outcome predicate differs between the all-realised and all-reference
assignments, the certificate exists and every graded variable is a deviation
signed by its owner; and (ii) \emph{groundedness}: a culpable-framing
finding against principal $p$, a positive grade on a deviation of $p$'s that
itself breaches the registered roster, is witnessed by $p$'s own certified
conduct. Principals with no deviating conduct receive no grade at any order,
and principals whose deviations are duty-compliant appear in the involvement
layer only. The involvement and culpability separation of
Theorem~\ref{thm:grounded} is therefore preserved under meta-lifting at
every order. Proof in Supplementary Note~3.
\end{theorem}

\begin{proposition}[Dilution is not free]
\label{prop:dilution}
Padding the deviation set with causally inert deviations changes neither the
existence of the certificate nor any grade: minimal witness sets exclude
variables on which the outcome predicate does not depend. Contrapositively,
any deviation whose inclusion changes a grade is causally active conduct,
certified and signed, and itself subject to Theorem~\ref{thm:hierarchy}.
Proof in Supplementary Note~3.
\end{proposition}

Theorem~\ref{thm:hierarchy} answers the natural question about
Theorem~\ref{thm:blameshift}: whether the certificate can itself be turned
into a weapon, by scapegoating a principal whose deviation was benign, or by
flooding the analysis with decoys. It cannot, and both halves are exercised
empirically in Results by a level-2 attack construction.
Theorem~\ref{thm:barrier} bounds what \emph{any} auditor confined to the
record, human or algorithmic, can certify, and predicts in particular that
attribution without ground-truth access should reach parity with, not
superiority over, ground-truth-informed judges on culprit identification.

\subsection*{From causes to verdicts}
Causation alone is not culpability. The verdict layer maps the certified
causal analysis into the four-element structure that negligence doctrine has
used for a century\cite{wright1985,harthonore1985}. \emph{Duty} asks which
declared obligations, such as safety envelopes, review requirements, and halt
authorities, were active for each principal in the window. \emph{Breach} asks
which certified messages, or duty-indexed absences, violated them, assessed
under both a specification standard and a reasonable-agent standard, each
declared and content-determined in the sense of Definition~\ref{def:cdd}, as
Theorem~\ref{thm:grounded} requires. \emph{Causation} is the certified cause
sets and grades of Eqs.~\eqref{eq:rho}--\eqref{eq:rhop}. \emph{Harm} is the
declared $Y$ with its measured severity. The output for each principal is the quadruple
(duty findings, breach findings, $\rhointerval_p$ with receipts, conduit and
exoneration flags), illustrated in Figure~\ref{fig:verdict}. The verdict is an \emph{input} to insurers, regulators,
and courts: an evidence-linked account, not an apportionment of legal
liability, which involves doctrine and discretion beyond causal
structure\cite{friedenberg2019}. And breach without causation is reported as
exactly that (the negligent-but-inert principal is named for the breach and
exonerated for the outcome), which single-culprit formats cannot express.

\subsection*{Implementation and models}
The implementation is in Python: Ed25519 message
signing\cite{bernstein2012}, Merkle epoch commitments\cite{merkle1988}, a
deterministic discrete-event facility simulator, structural-model construction
with exhaustive gate-and-grade search over incident slices, and replay by
seeded re-execution. Language models play three distinct
roles: as \emph{subjects}, they are the planner, solvers, aggregator, and
verifier whose failures are audited; as the \emph{baseline}, the same model
runs the field's judge formats; and as \audita{}'s \emph{attributor}, a model
reasons over the certified record only where the degraded static-log variant
applies. All served models are open-weight (Qwen2.5-72B-Instruct as the
primary model, with Qwen2.5-7B-Instruct and Llama-3.1-8B-Instruct for additional study), run under vLLM\cite{kwon2023vllm} with tensor-parallel
serving on NVIDIA A100-80GB GPUs, greedy decoding (temperature $0$) with a
fixed seed, and the serial replay protocol that makes receipts bitwise exact.
Each experiment uses three generation seeds. 

% audita-results.tex — Results section (Nature order: Results -> Discussion -> Methods).
% Narrative: benchmark head-to-head -> contributed dataset -> robustness -> ablations ->
% implementation. All numbers trace to archived artifacts (verify_paper_numbers.py).

\section*{Results}

We evaluate \audita{} on two families of failures. The first is a live
corpus, a multi-agent pipeline solving public GSM8K
problems\cite{cobbe2021gsm8k} whose failures arise from real model behaviour;
the second is CulpaBench, our benchmark of physical incident structures in a
deterministic multi-robot facility, with responsibility ground truth computed
by exact counterfactual re-execution. We compare \audita{} against the
field's attribution baselines: the large-language-model judge of the
Who\&When benchmark\cite{zhang2025who} (all three published formats, with the
live corpus running all-at-once), single-site counterfactual scoring as
published in CAR and CausalFlow\cite{shah2026car,bonagiri2026causalflow},
Shapley-value attribution\cite{ma2025causal}, and, on the benchmark
structures, a statistical anomaly detector. The primary metric is
responsibility error, the $\ell_1$ distance between the reported and the
true graded responsibility profile (lower is better); on the external
Who\&When benchmark we additionally report that benchmark's own agent- and
step-accuracy.

\subsection*{Attribution on live multi-agent failures}
The first question is whether \audita{} attributes real failures better than
the methods the field currently uses. The pipeline has four roles: a planner, three parallel
solvers, an aggregator that carries a declared normative duty, and a verifier.
Three independently generated corpora ($480$ cases each, one development
corpus and two holdouts) supply the incidents.
Fault injection is disclosed and calibrated to the empirical multi-agent
failure taxonomy\cite{cemri2025mast}, and every label is replay-verified.

Re-deriving ground truth from causal structure rather than from injection
bookkeeping produced an unexpected finding. We predicted that joint causation, a structure prior
methods\cite{zhang2025who,shah2026car,bonagiri2026causalflow} cannot
express, would account for very few single-fault incidents. Instead, under duty-aware
graded ground truth, $95$--$98\%$ of adverse incidents ($56/57$, $43/47$,
$39/41$ across the three corpora) are joint solver--aggregator incidents, and
not one is the monotone single-culprit structure the field's metric presumes.
An aggregator holding a declared duty is a but-for co-cause of nearly every
adverse outcome that duty would have prevented, so the single-culprit label,
which names only the erring solver, discards a co-cause in almost every
incident. The field's own benchmark shows the same pattern: in $58.7\%$ of
Who\&When traces a duty-bearing role (a verifier or orchestrator, by the
benchmark authors' own role names) stayed silent through the decisive error
and is not the labelled culprit.

Results are shown in Table~\ref{tab:live}. The full graded verdict, which reads the correct answer,
reduces the judge's responsibility error from $0.589$ to $0.170$. And the
deployable variant attains the lowest error in the table, $0.159$, while
reading no answer key at all. The reason lies in what the answer key is used for. The two breach tests
agree whenever an agent violated a registered duty; they can disagree
essentially in one place, the \emph{honest error}, where a solver follows
every duty, reasons consistently, and still lands on a wrong answer. There
the gold-informed test, seeing the wrong answer, declares a breach. Meanwhile, the deployable checks conduct only against the registered duties. The ground truth sides with the latter: under negligence
semantics an agent that violated no duty is not culpable. Among the baselines,
Shapley-value attribution is the strongest at $0.223$ and single-site
scoring reaches $0.408$; causal-only, \audita{} stripped of its duty layer,
also stops at $0.408$, so the duty layer is what separates the verdict from
the field.

\begin{table}[t]
\centering
\caption{\textbf{Attribution error on live multi-agent failures} (GSM8K
pipeline; responsibility error against duty-aware Chockler--Halpern ground
truth). Baselines are the
LLM judge of the Who\&When benchmark, CAR~/ CausalFlow
single-site scoring, and Shapley-value attribution. \audita{} appears in
three variants that differ only in what the verdict uses.
\emph{Causal-only} removes the duty layer entirely and reports graded
causation alone. \emph{Full} is the complete duty--breach--causation verdict,
whose breach test may consult the gold answer, the problem's correct final
answer, when deciding whether an agent's output was wrong.
\emph{Deployable} is the same verdict except that its breach test never sees
the correct answer and judges each agent only against its own certified
messages (Definition~\ref{def:cdd}), so it can run in production, where no
answer key exists. Errors are mean $\pm$ s.d.\
across the three independently generated corpora. The full and deployable
verdicts attain the two lowest errors, roughly a third of the judge's.}
\label{tab:live}
\small
\begin{tabular}{@{}lc@{}}
\toprule
Method & Responsibility error $\downarrow$ \\
\midrule
LLM judge (all-at-once)\cite{zhang2025who} & $0.589 \pm 0.017$ \\
CAR~/ CausalFlow (single-site)\cite{shah2026car,bonagiri2026causalflow} & $0.408 \pm 0.020$ \\
Shapley-value attribution\cite{ma2025causal} & $0.223 \pm 0.042$ \\
\midrule
\audita{} (causal-only, duty layer ablated) & $0.408 \pm 0.020$ \\
\audita{} (full, gold-informed) & $0.170 \pm 0.035$ \\
\audita{} (deployable, no answer key) & $\mathbf{0.159 \pm 0.030}$ \\
\bottomrule
\end{tabular}
\end{table}

\subsection*{CulpaBench: a benchmark of causal-attribution structures}
Certifying what a method recovers requires planted, analytically derived
ground truth. We therefore contribute \emph{CulpaBench}, a benchmark of
physical incident structures that provides it, so estimator correctness is
measured against causal structure instead of agreement with our own analysis. Its structures are grounded in $42$ public robot-accident records
(OSHA and NIOSH reports; Supplementary Information), its scale of over 1,000 incidents is roughly $10 \times$ the $184$-trace field benchmark, and
we release it as a resource for attribution
research. Extended Data Table~\ref{tab:scenarios} lists the
structures; each targets a specific way single-culprit attribution fails.

Evaluation results are shown in Table~\ref{tab:scenario_results}. The certified
graded verdict (Figure~\ref{fig:verdict}) recovers every structure exactly
(responsibility error $0.000$, mean over $300$ draws per structure),
including the $\tfrac12,\tfrac12$ split of an overdetermined accident and
the omission graded through its absence variable. The baselines fail
structurally and predictably. Single-site scoring cannot represent
overdetermination (deleting either sufficient cause alone leaves the outcome
standing), cannot see an omission, and cannot separate the inert twin, a
command recorded identically to a real cause but whose effect was locked out
before reaching the world. The anomaly detector always has a most-anomalous
message, so it can never say ``no one is responsible''. Shapley-value
attribution is graded but blind to duty, spreading mass across non-culpable
candidates. Language-model judges clear the trivial control exactly but err
wherever a cause is shared or absent: a single-culprit output cannot express
either, however well the model reasons.

\begin{table}[t]
\centering
\caption{\textbf{Attribution on the CulpaBench structures} (responsibility
error; mean over $300$ draws per
structure). Cells are mean $\pm$ s.d.\ across the three
generation seeds. Some LLM-free arms are
seed-invariant, so their s.d.\ is identically zero. The Average column is the
mean of the five structure means. The LLM judge, CAR~/
CausalFlow single-site scoring, Shapley-value attribution, and the anomaly
detector are the baselines; causal-only is \audita{} with
the duty layer ablated. \audita{}'s exact-match rate is $1.000$ in
every structure.}
\label{tab:scenario_results}
\small
\resizebox{\textwidth}{!}{%
\begin{tabular}{@{}lcccccc@{}}
\toprule
Method & Redundancy $\downarrow$ & Inert twin $\downarrow$ & Preemption $\downarrow$ & Omission $\downarrow$ & Chain $\downarrow$ & \textbf{Average} $\downarrow$ \\
\midrule
\audita{} (full, certified) & $\mathbf{0.000 \pm 0.000}$ & $\mathbf{0.000 \pm 0.000}$ & $\mathbf{0.000 \pm 0.000}$ & $\mathbf{0.000 \pm 0.000}$ & $\mathbf{0.000 \pm 0.000}$ & $\mathbf{0.000}$ \\
\audita{} (causal-only, duty layer ablated) & $0.667 \pm 0.000$ & $0.667 \pm 0.000$ & $0.000 \pm 0.000$ & $0.500 \pm 0.000$ & $0.000 \pm 0.000$ & $0.367$ \\
\midrule
LLM judge (all-at-once)\cite{zhang2025who} & $0.500 \pm 0.000$ & $0.017 \pm 0.029$ & $0.000 \pm 0.000$ & $0.500 \pm 0.000$ & $0.500 \pm 0.000$ & $0.303$ \\
LLM judge (step-wise)\cite{zhang2025who} & $0.500 \pm 0.000$ & $0.000 \pm 0.000$ & $0.017 \pm 0.029$ & $0.500 \pm 0.000$ & $0.500 \pm 0.000$ & $0.303$ \\
LLM judge (binary-search)\cite{zhang2025who} & $0.567 \pm 0.038$ & $0.150 \pm 0.050$ & $0.000 \pm 0.000$ & $0.500 \pm 0.000$ & $0.500 \pm 0.000$ & $0.343$ \\
CAR~/ CausalFlow (single-site)\cite{shah2026car,bonagiri2026causalflow} & $1.000 \pm 0.000$ & $0.667 \pm 0.000$ & $0.000 \pm 0.000$ & $0.500 \pm 0.000$ & $0.000 \pm 0.000$ & $0.433$ \\
Shapley-value attribution\cite{ma2025causal} & $0.833 \pm 0.000$ & $0.667 \pm 0.000$ & $0.095 \pm 0.000$ & $0.667 \pm 0.000$ & $0.000 \pm 0.000$ & $0.452$ \\
Statistical anomaly detector & $1.000 \pm 0.000$ & $1.000 \pm 0.000$ & $1.000 \pm 0.000$ & $1.000 \pm 0.000$ & $0.500 \pm 0.000$ & $0.900$ \\
\bottomrule
\end{tabular}}
\end{table}

% audita-fig4.tex — Fig. 2: the verdict schema and what it can express.
\begin{figure}[t]
\centering
\resizebox{\textwidth}{!}{%
\begin{tikzpicture}[font=\scriptsize,
  leg/.style={draw=inkP,fill=fillP,rounded corners=2pt,align=center,minimum height=18mm,text width=30mm,inner sep=2pt,line width=0.7pt},
  fl/.style={-{Stealth[length=1.8mm]},inkG,line width=0.6pt}]

% ===== type signature =====
\node[draw=inkP,fill=white,line width=1.1pt,rounded corners=3pt,align=center,text width=142mm,minimum height=9mm] (V) at (7.4,5.1)
  {\textbf{Verdict}$(p)\;=\;\big\langle\, \mathrm{duty}\!\in\!\{0,1\},\;\; \mathrm{breach}\!\in\!\{\varnothing,\,\sigma_{\mathrm{spec}},\,\sigma_{\mathrm{ra}}\},\;\; \rhointerval+\text{receipts},\;\; \mathrm{flags}\,\big\rangle$};

% ===== four legs, equal height, straight vertical arrows =====
\node[leg] (d) at (1.8,3.05) {\textbf{Duty}\\[2pt] $\mathrm{duty}(p)=[\,p\!\in\!\mathrm{dom}(D)\,]$\\[2pt] {\color{inkG}delegation \&\\ duty registry $D$}};
\node[leg] (b) at (5.55,3.05) {\textbf{Breach}\\[2pt] $R(p)\not\models \sigma$\\[2pt] {\color{inkG}receipts $R(p)$ vs.\ standards $\sigma_{\mathrm{spec}}$, $\sigma_{\mathrm{ra}}$}};
\node[leg] (c) at (9.3,3.05) {\textbf{Causation}\\[1pt] $\rho_p=\dfrac{1}{|X^\ast|+|W^\ast|}$\\[1pt] {\color{inkG}minimal cause $X^\ast$,\\ witness $W^\ast$; replay-certified}};
\node[leg] (h) at (13.05,3.05) {\textbf{Harm}\\[2pt] $\mathrm{sev}(Y)$\\[2pt] {\color{inkG}recorded terminal\\ state \& severity}};
\foreach \x in {d,b,c,h} \draw[fl] (\x.north) -- (\x.north |- V.south);

% ===== breach x causation truth table (single-line grid) =====
\node[font=\bfseries,text=inkP] at (7.4,1.55) {Breach $\times$ causation: what the quadruple can express};
\node[font=\scriptsize,text=inkT] at (5.65,0.85) {causation: $\rho_p>0$};
\node[font=\scriptsize,text=inkT] at (9.15,0.85) {no causation: $\rho_p=0$};
\node[font=\scriptsize,text=inkP,rotate=90,anchor=center] at (3.55,-0.10) {breach};
\node[font=\scriptsize,text=inkP,rotate=90,anchor=center] at (3.55,-1.40) {no breach};
% cell fills first, then one clean grid on top
\fill[fillcert] (3.9,0.55) rectangle (7.4,-0.75);
\fill[fillsusp] (7.4,0.55) rectangle (10.9,-0.75);
\draw[inkA,line width=1.1pt] (7.4,0.55) rectangle (10.9,-0.75);
\draw[inkG,line width=0.6pt] (3.9,0.55) rectangle (10.9,-2.05);
\draw[inkG,line width=0.6pt] (7.4,0.55) -- (7.4,-2.05);
\draw[inkG,line width=0.6pt] (3.9,-0.75) -- (10.9,-0.75);
% cell texts (no boxes)
\node[align=center,text width=30mm] at (5.65,-0.10) {\textbf{culpable}\\ cited \& responsible\\ ($\rho_p>0$)};
\node[align=center,text width=30mm] at (9.15,-0.10) {\textbf{cited, exonerated}\\ negligent-but-inert};
\node[align=center,text width=30mm] at (5.65,-1.40) {\textbf{involved, not culpable}\\ e.g.\ a faithful conduit\\ (Thm.~\ref{thm:grounded})};
\node[align=center,text width=30mm] at (9.15,-1.40) {\textbf{cleared}\\ no breach, no cause};
% one-line note under the table
% \node[font=\tiny\itshape,text=inkA] at (7.4,-2.45)
%   % {the amber cell, breach without causation, is expressible here (Thm.~\ref{thm:exon}) and impossible for a single-culprit label};

\end{tikzpicture}}
\caption{\textbf{The verdict schema and what it can express.} \textit{Top.} \audita{}'s output
is not a single culprit but a typed verdict per principal $p$: a duty bit, a
breach label assessed under two standards, a graded responsibility interval
with receipts, and conduit or exoneration flags. Each leg is computed
from a concrete record artifact and a formal object: duty from the
machine-readable delegation and duty registry $D$; breach from the receipts
$R(p)$ tested against both a specification standard $\sigma_{\mathrm{spec}}$ and
a reasonable-agent standard $\sigma_{\mathrm{ra}}$; causation from the graded
score $\rho_p=1/(|X^\ast|+|W^\ast|)$ over the minimal cause set $X^\ast$ and
witness set $W^\ast$, certified by replay (Eq.~\eqref{eq:rho}); and harm from
the recorded terminal state and the declared severity of $Y$. \textit{Bottom.} Because breach
and causation are separate axes, the quadruple expresses cases a single-culprit
label cannot: a principal in breach but not a cause is cited and
simultaneously exonerated for the outcome, the negligent-but-inert principal,
while a principal that is a cause but in no breach is reported as involved yet
never convicted, faithful relays carrying the \emph{conduit} flag; that
separation is what Theorem~\ref{thm:grounded} guarantees.}
\label{fig:verdict}
\end{figure}

The transcript-identical twin isolates what the certified evidence buys: the
redundancy and inert-twin structures present bitwise-identical records and
differ only in whether the second command's effect physically reached the
world, so every transcript-level method returns one verdict for both and is
right on one structure and wrong on the other, while the certified graded
engine separates them in all draws. The structures also compose: on a sampled
distribution of $1{,}002$ incidents drawing redundancy width, delegation
depth, omission placement, and the adverse predicate independently ($57.8\%$
adverse), \audita{} attains error $0.000$ and exact cause-set match $1.000$ on
the $579$ adverse draws, against single-site's $0.370$.

\emph{Negative controls.} The most falsifying test is the accident with no
culprit. On $300$ no-cause draws \audita{} manufactures a culprit in zero
cases. On the same incidents the anomaly detector manufactures one in
every case, and language-model judges name someone in $100\%$ of runs across
all three formats, under a neutral prompt that explicitly offers ``NONE''. A
judge shown a genuine accident always blames a principal; \audita{} does not.

\emph{Generalisation.} With the engine, oracle and experimental arms frozen, we conducted an out-of-loop evaluation on an unexpected-startup structure, the modal fatal pattern in the accident
corpus\cite{layne2023robot}: a worker is struck during maintenance because
the pre-entry halt was never issued, so every command is duty-compliant and
the sole breach is the omission. On $74$ held-out incidents \audita{} is
exact ($0.000$) and blames the omission alone; single-site pins the innocent
sole actor every time, and judges convict an innocent in $67$--$100\%$ of
runs unless the silent duty-holder is offered as a candidate. \audita{}
needs no such offering: the culprit is the principal who did nothing.

\subsection*{Ablations and analysis}
Removing the duty layer breaks the structures where involvement and
culpability diverge (redundancy and inert twin: error $0.667$ in
Table~\ref{tab:scenario_results}; live error rises from $0.170$ to $0.408$),
while the production gate is load-bearing on preemption and the inert twin,
where it retires commands whose effects never reached the outcome. The
deployable breach standard matches the
gold-informed one on live data ($0.159$ versus $0.170$). This confirms that the
completeness barrier (Theorem~\ref{thm:barrier}) bounds what an auditor can
\emph{certify}.

Verdicts are relative to the
registered duty standard: holding causal facts fixed and varying only the
standard leaves the causal profile identical but flips
culpability in $69.8\%$ of them, \textbf{so a verdict must ship with its standard
attached}, which no prior method does because none has one.

The advantage
transfers across model families, a three- to fivefold gap over the judge on
both Qwen2.5-7B and Llama-3.1-8B (Figure~\ref{fig:results}a). The comparison against Shapley-value attribution,
the strongest baseline, depends on scale. On the
weakest model Shapley is the more accurate of the two ($0.041$ versus $0.111$). The graded
verdict overtakes it on the 8B and 72B models, and the margin widens with
model capability from $-0.070$ to $+0.037$ to $+0.053$. Shapley is also alone
among the four baselines in carrying no duty verdict, no certificate, and no
resistance to the record forgery (below).

% Three further measurements bound what the method
% costs and what it misses. Exact recovery holds at every tested incident width
% while compute doubles per added candidate (Extended Data
% Table~\ref{tab:cost}), so the limit is computational rather than statistical,
% and it falls near $k\approx8$ at a one-second budget. The record's perimeter is
% measured: when $31\%$ of coordination is deliberately moved
% off the recorded bus, $0.157$ of responsibility becomes invisible, yet
% over-attribution stays $0.000$ across all $1{,}200$ draws, so what the auditor
% loses is coverage and never soundness. Repair guided by the verdict averts
% $2.4\times$ the physical harm of a budget-matched random fix at a single edit. Finally, replaying a served language model is
% bitwise exact only when the replays are executed one at a time: a parallelised
% scoring pass once fabricated an innocence violation purely through
% floating-point drift in the batched reduction order. Serial replay eliminates
% it, so serial execution is part
% of the pinned-stack definition, not an optimisation. The deterministic
% facility benchmark is immune by construction.

% audita-fig-results.tex — Fig. 3: scaling + adversarial robustness (two panels).
% Rendered by scripts/gen_fig3.py -> audita-fig3-results.pdf (data from archived artifacts).
\begin{figure}[t]
\centering
\includegraphics[width=\textwidth]{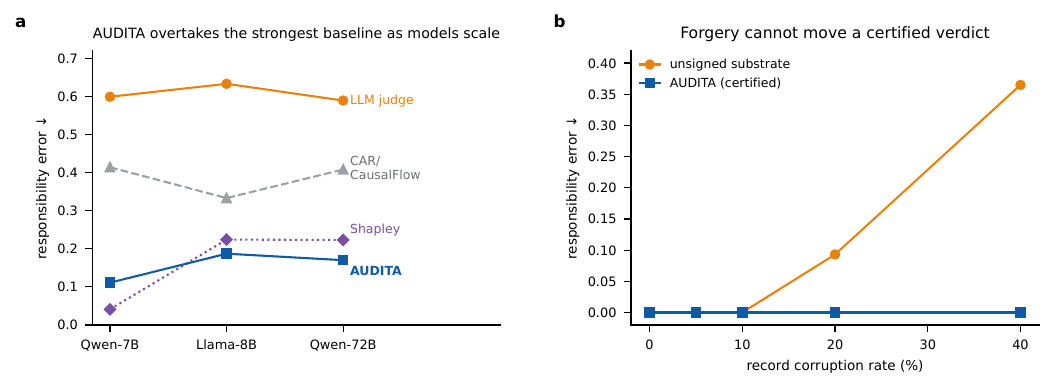}
\caption{\textbf{Scaling and adversarial robustness.}
\textbf{a}, Four methods across model families (Qwen2.5-7B, Llama-3.1-8B,
Qwen2.5-72B). The gap over the LLM judge (all-at-once) and the single-site baseline
(CAR/CausalFlow) holds at every scale. Shapley-value attribution is the exception on the weakest model: on Qwen2.5-7B it attains
$0.041$ against \audita{}'s $0.111$, but \audita{} overtakes it on the 8B and 72B
models, its margin growing from $-0.070$ to $+0.037$ to $+0.053$ as model
capability rises. \textbf{b}, Under record forgery at rising corruption rates,
an unsigned substrate is driven to $0.365$ error and frames an innocent victim
in all probes, while the certified verdict stays exact.}
\label{fig:results}
\end{figure}

\subsection*{Robustness to tampering and manipulation}
An audit layer must survive an adversary. We test three attack classes:
record forgery and deletion (forgery summarised in
Figure~\ref{fig:results}b), volume manipulation, and insiders with
valid keys attacking through conduct.

\emph{Record attacks.} We replay the same forgeries and deletions against two
substrates: the certified record, and an unsigned copy that keeps every
message's text but drops the signatures and seals. On the certified
substrate every forged item fails verification and never enters the
certified evidence, so the verdict is bitwise unchanged (error $0.000$); on
the unsigned substrate the same forgeries drive error to $0.365$ and frame
an innocent victim in all $450$ probes. Deletion cannot be blocked, but the
verdict's output format absorbs it: \audita{} reports responsibility as an
interval over all completions of the record consistent with the surviving
evidence (Methods), so deleting certified evidence widens the interval rather than
shifting a single number. Measured across all $450$ deletion
cells, the widened interval always contains the true responsibility, the
verdict's error stays at most $0.122$, and every deletion is flagged with
its author. Erasure costs the record precision, never correctness, and never
silently exonerates.

\emph{Volume manipulation.} Flooding the log with an innocent principal's causally inert commands
cannot change \audita{}'s verdict: an inert command has no
production path into the outcome, so no amount of flooding ever creates a
candidate, and the flooded principal's grade stays zero by construction. The
baselines turn out to be equally immune: regressing the attribution mass
each method places on the flooded principal against the flooding level
yields a slope of $0.000$ in every arm. Volume manipulation is therefore a
dead end against certified attribution and baselines alike.

\emph{Conduct attacks by valid-key insiders.} The strongest adversary in the
threat model holds valid keys and never touches the record: it attacks
through its own signed conduct during the incident. We implemented the two
attack constructions of Methods and ran them over $600$ incidents. Both
succeed at the one thing no record machinery can prevent: the attacker makes
a fully compliant victim a genuine cause of the harm, raising the victim's
causal involvement to $1.0$ in all $600$, with every message verifying. What
the attacks cannot do is convict the victim. Its culpability stays $0.000$
in all $600$, because a culpability finding must be witnessed by the
accused's own certified breach (Theorem~\ref{thm:grounded}), and in every
incident the audit also emits a machine-checkable certificate naming the
attacker and grading its contribution to the framing
(Theorem~\ref{thm:blameshift}). We then attacked the certificate itself: a
construction that scapegoats a benign deviation convicts it in $0$ of $330$
draws while naming the true attacker in all $330$
(Theorem~\ref{thm:hierarchy}), and padding the analysis with inert decoys
changes no grade (Proposition~\ref{prop:dilution}).

\subsection*{Costs and limits}
\label{sec:A2}
The remaining measurements price what the method costs and what it cannot
see. Exact recovery holds at every tested incident width while compute
doubles per added candidate (Extended Data Table~\ref{tab:cost}), so the
limit is computational rather than statistical, and it falls near
$k\approx8$ at a one-second budget. The record's perimeter is measured:
when $31\%$ of coordination is deliberately moved off the recorded bus,
$0.157$ of responsibility becomes invisible, yet over-attribution stays
$0.000$ across all $1{,}200$ draws, so what the auditor loses is coverage
and never soundness. The verdict is also actionable: repair guided by it
averts $2.4\times$ the physical harm of a budget-matched random fix at a
single edit.

\subsection*{Comparability on the field benchmark}
\audita{}'s attribution advantage
requires counterfactual access, which live pipelines and the facility provide
but a static log does not. On the Who\&When benchmark\cite{zhang2025who}, whose
$184$ traces are static logs that cannot be re-executed, \audita{}'s
counterfactual and cryptographic machinery does not apply; what remains is a
duty- and structure-aware attributor that reads the trace alone. We evaluate
this degraded variant on Who\&When's own task and metric, against the same
model running the field's judge formats (Table~\ref{tab:whowhen}). It performs
comparably: agent-accuracy $0.408$ against the judge band of $0.41$--$0.51$
(the strongest judge format, step-wise, reaches $0.511$ and reproduces the
published $0.535$ within sampling error), and step-accuracy $0.321$, above two
of the three judge formats. We claim non-inferiority here: where the log cannot be
re-executed, \audita{} matches the field's methods, and its advantage appears
only where counterfactual access exists.

\begin{table}[t]
\centering
\caption{\textbf{Comparability on the Who\&When benchmark} ($184$ static
traces; the benchmark's own agent- and step-accuracy;
Qwen2.5-72B for every arm). \audita{}-structural is the degraded variant
that reads the static trace alone, without counterfactual replay or
cryptographic verification, matching what the benchmark provides. Stripped of its machinery,
\audita{}-structural stays within the judge band ($0.41$--$0.51$) on
agent-accuracy and exceeds two of the three formats on step-accuracy: where
the log cannot be re-executed, \audita{} matches the field, and its
advantage (Tables~\ref{tab:live} and~\ref{tab:scenario_results}) appears
exactly where counterfactual access exists.}
\label{tab:whowhen}
\small
\begin{tabular}{@{}lcc@{}}
\toprule
Method & Agent-accuracy $\uparrow$ & Step-accuracy $\uparrow$ \\
\midrule
LLM judge (all-at-once)\cite{zhang2025who} & $0.413$ & $0.255$ \\
LLM judge (step-wise)\cite{zhang2025who} & $0.511$ & $0.304$ \\
LLM judge (binary-search)\cite{zhang2025who} & $0.418$ & $0.337$ \\
\midrule
\audita{}-structural & $0.408$ & $0.321$ \\
\bottomrule
\end{tabular}
\end{table}

% audita-discussion.tex — trimmed for Nature MI.

\section*{Discussion}

\audita{} supplies what multi-agent systems lack: an evidentiary and causal
account of a concrete adverse event, produced by machinery in place and checkable by a hostile audience afterwards. Auditing the event
requires both an evidence substrate and a causal calculus, and the central
claim is that the two are useful only together: semantics without certified
evidence can be framed by whoever edits the log, and certified evidence without
causal semantics reproduces the single-culprit errors that redundancy,
preemption, and omission induce. The evaluation tests the two separately, so the composition claim is
falsifiable.

Together the results map a tight possibility frontier for post-incident
accountability: what the theorems forbid, attacks realise; what they permit,
the running system achieves. The impossibilities bind any record-confined
auditor: an attacker controlling the record can force widened uncertainty, a
valid-key insider can make an innocent principal a genuine cause of harm,
and no such auditor can certify all outcome-guilt without solving the
audited task (Theorem~\ref{thm:barrier}). The composition prevents
\emph{manufactured certainty} and \emph{silent innocence}: culpability moves
only on the accused's own certified conduct (Theorem~\ref{thm:grounded}),
shifted blame is graded by the calculus it abused
(Theorem~\ref{thm:blameshift}), and erased evidence yields flagged,
author-attributed uncertainty, never certified innocence
(Theorem~\ref{thm:exon}). The design does not assume the presumption of
innocence; it proves it for certified records. The barrier also predicts the
parity with gold-informed judges on culprit identification that the
field-benchmark comparison shows.

The guarantees are scoped by three assumptions, each paired with a measured
price of violation. Under a \emph{replayable stack} receipts are exact; on
hosted non-deterministic models they degrade to statistical corroboration, so
exact accountability favours replayable deployments; we read this as a fact
about accountability itself. Under a \emph{recorded channel} exonerations
are complete; coordination through side channels is invisible, so the verdict
clears a principal only ``on the certified evidence, to residual $r$,'' and we measure $r$ under
deliberate off-bus coordination instead of assuming it is small. Under
\emph{sound key custody} authorship equals conduct; the assumption concedes
nothing else, since insiders holding valid keys may be arbitrarily malicious,
and the theorems state what they still cannot achieve.

Two boundaries of the formalism remain. Incidents are discrete, a predicate
that becomes true at a time, so cumulative harms without a threshold event are
out of scope. And the quantitative evidence is simulated: the structures are
grounded in public accident reports and the framework is domain-agnostic, but
we do not yet demonstrate physical robots or production stacks.
Analytically-derived ground truth from an independent oracle, and scoring
against a re-executed world, are what make the simulated evidence meaningful; an
embodied simulator with contact dynamics and live planners is the natural next
step.

The path outward follows the regulation that motivated the design: the EU AI
Act requires that records exist\cite{euaiact2024}; \audita{} proposes what they
must \emph{be}, signed, cited, sealed, and causally analysable, for those
obligations to purchase accountability. The layer ports beyond factories to
logistics fleets, laboratory automation, service robots, and digital agent
economies, where the adverse outcome is a corrupted transaction. As AI brains take command of physical fleets, ``who did what, and
how much did it matter?'' will be asked with growing force, and \audita{} is
our proposal for answering it with evidence.

% audita-backmatter.tex — back matter (shared by main and combined files).
% ---------- back matter ----------
\section*{Data availability}
CulpaBench, its ground-truth specifications, the generated live corpora, and
all experiment artifacts (result files for every run, indexed by experiment
identifier and seed) are available at
\url{https://github.com/ZhixuDu/audita}. 

\section*{Code availability}
The complete \audita{} reference implementation, comprising the accountable
record substrate, the attribution engine, the scenario injectors, all
baselines, and the experiment runners, is available under the MIT license at
\url{https://github.com/ZhixuDu/audita}.

\section*{Author contributions}
Z.D. conceived the method, designed and performed the experiments, analysed
the results, and wrote the manuscript. Y.C. supervised the project and
reviewed the manuscript.

\section*{Competing interests}
The authors declare no competing interests.

% ===================== REFERENCES =====================
\bibliography{audita-refs}

% ===================== EXTENDED DATA =====================
\clearpage
% audita-extended-data.tex — Extended Data figures and tables (Nature MI).
% Overflow display items moved from the main text; up to 10 allowed.
\clearpage
\setcounter{figure}{0}\setcounter{table}{0}
\renewcommand{\thefigure}{\arabic{figure}}
\renewcommand{\thetable}{\arabic{table}}
\renewcommand{\figurename}{Extended Data Fig.}
\renewcommand{\tablename}{Extended Data Table}
\section*{Extended Data}

% audita-fig2.tex — Fig. 2: anatomy of one accountable message
\begin{figure}[H]
\centering
\begin{tikzpicture}[font=\scriptsize,
  row/.style={draw=inkG,fill=white,minimum height=5.2mm,text width=68mm,inner xsep=4pt,align=left},
  rowN/.style={draw=inkA,fill=fillA!60,minimum height=5.2mm,text width=68mm,inner xsep=4pt,align=left},
  stamp/.style={draw=inkT,fill=fillT,rounded corners=2pt,minimum width=21mm,minimum height=9mm,align=center,font=\scriptsize},
  lbl/.style={font=\scriptsize\itshape,inkG}]
\node[lbl] (cap) at (0,0.55) {a signed message $m=(a,\pi,C,\varphi,\sigma)$ (every field covered by $\sigma$)};
\node[row,below=1mm of cap.south,anchor=north] (f1) {\textbf{id} $m_{47}$ \quad\textbf{author} $a=\text{planner-A}$ \quad\textbf{dst} robot-3 \quad\textbf{t} $1024.6$};
\node[row,below=0.6mm of f1] (f2) {\textbf{payload} $\pi$ \;\texttt{move(aisle\_2, v=0.8)}};
\node[rowN,below=0.6mm of f2] (f3) {\textbf{cites} \;$C=\{m_{31}, m_{44}\}$ \hfill $\leftarrow$ authored lineage};
\node[rowN,below=0.6mm of f3] (f4) {\textbf{effect} \;$\varphi = $ commitment to the state change \hfill $\leftarrow$ replay checkpoint};
\node[rowN,below=0.6mm of f4] (f5) {\textbf{seed/model} \;$(\theta, \xi)$ \hfill $\leftarrow$ reproducible replay};
\node[row,below=0.6mm of f5,draw=inkP,fill=fillP] (f6) {\textbf{sig} \;$\sigma_{a}(a,\pi,C,\varphi)$};
\node[draw=inkP,line width=0.9pt,rounded corners=3pt,fit=(f1)(f2)(f3)(f4)(f5)(f6),inner sep=2pt] (env) {};
\node[lbl,below=1mm of env.south,text width=72mm,align=center] {amber fields are absent from standard agent logs and are exactly what the attribution engine consumes};
\begin{scope}[yshift=-52mm]
\node[stamp] (s1) at (-2.35,0)  {1. Write\\\emph{all fields}};
\node[stamp,right=3mm of s1] (s2) {2. Sign\\\emph{author key}};
\node[stamp,right=3mm of s2] (s3) {3. Cite\\\emph{or reject}};
\node[stamp,below=8mm of s1] (s4) {4. Receipt\\\emph{co-sign}};
\node[stamp,right=3mm of s4] (s5) {5. Seal\\\emph{Merkle epoch}};
\node[stamp,right=3mm of s5] (s6) {6. Store\\\emph{content-addressed}};
\draw[-{Stealth[length=1.6mm]},inkG] (s1)--(s2);
\draw[-{Stealth[length=1.6mm]},inkG] (s2)--(s3);
\draw[-{Stealth[length=1.6mm]},inkG] (s3.south) .. controls +(0,-4mm) and +(0,4mm) .. (s4.north);
\draw[-{Stealth[length=1.6mm]},inkG] (s4)--(s5);
\draw[-{Stealth[length=1.6mm]},inkG] (s5)--(s6);
\draw[-{Stealth[length=1.6mm]},inkA,dashed] (s4.west) .. controls +(-6mm,0) and +(-6mm,0) .. (s1.west)
  node[midway,left,font=\tiny,text=inkA]{ack};
\node[lbl,below=2mm of s5.south,text width=70mm,align=center] {six stamps per message; the receipt (dashed) makes a dropped message detectable};
\end{scope}
\end{tikzpicture}
\caption{\textbf{Anatomy of one accountable message and its lifecycle.} Top:
the fields of a single message (Definition~\ref{def:graph}); the three amber
fields, the authored citation set $C$, the committed effect predicate
$\varphi$, and the pinned seed/model identifiers, are absent from standard
agent-observability logs and are exactly what gate, grade, and replay
consume. Bottom: the six stamps every message collects on the bus. Citation
and receipt are enforced (an uncited command is rejected; every delivery is
acknowledged), so the recorded graph is correct by construction and silent
drops become detectable, which is what moves deletions into the missing set
of Proposition~\ref{prop:recordedit} rather than out of history.}
\label{fig:anatomy}
\end{figure}

% audita-fig3.tex — Fig. 3: the two-layer attribution engine
\begin{figure}[H]
\centering
\resizebox{\textwidth}{!}{%
\begin{tikzpicture}[font=\scriptsize,
  io/.style={draw=inkG,fill=white,rounded corners=1pt,minimum width=24mm,minimum height=9mm,align=center,font=\scriptsize},
  intern/.style={draw=inkP,fill=fillP,rounded corners=2pt,minimum width=25mm,minimum height=9mm,align=center},
  internB/.style={draw=inkA,fill=fillA,rounded corners=2pt,minimum width=25mm,minimum height=9mm,align=center},
  eng/.style={draw,dashed,rounded corners=4pt,inner sep=3mm},
  fl/.style={-{Stealth[length=2mm]},inkG,line width=0.5pt}]

% Engine A internals
\node[intern] (a1) at (0,0)   {discretise slice\\\footnotesize $\to$ structural model};
\node[intern] (a2) at (3.1,0) {exact search over\\\footnotesize gated candidates};
\node[intern] (a3) at (6.2,0) {enumerate minimal\\\footnotesize $(X^\ast,W^\ast)$ pairs};
\draw[fl] (a1)--(a2); \draw[fl] (a2)--(a3);
\node[eng,draw=inkP,fit=(a1)(a2)(a3),label={[text=inkP,font=\footnotesize\bfseries]above:{Engine A: surrogate proposer}}] (EA) {};

% Engine B internals
\node[internB] (b1) at (9.6,0)  {freeze $W^\ast$ at\\\footnotesize recorded values};
\node[internB] (b2) at (12.7,0) {re-execute: seeded\\\footnotesize models + simulator};
\node[internB] (b3) at (15.8,0) {risk drop $\geq\delta$?\\\footnotesize exact / CI + match};
\draw[fl] (b1)--(b2); \draw[fl] (b2)--(b3);
\node[eng,draw=inkA,fit=(b1)(b2)(b3),label={[text=inkA,font=\footnotesize\bfseries]above:{Engine B: replay certifier}}] (EB) {};

% inputs / outputs row
\node[io] (ai) at (0,-2.3) {\textbf{in:} certified\\subgraph $\Gcert$};
\node[io,draw=inkP,minimum width=30mm] (ao) at (6.2,-2.3) {\textbf{out:} candidates +\\$\rho=1/(|X^\ast|{+}|W^\ast|)$};
\node[io,minimum width=30mm] (bi) at (9.9,-2.3) {\textbf{in:} $(X^\ast,W^\ast)$, $\varphi$\\checkpoints, seeds};
\node[io,draw=inkA,minimum width=28mm] (bo) at (15.8,-2.3) {\textbf{out:} certified causes,\\risk curves, receipts};
\draw[fl] (ai)--(a1);
\draw[fl] (a3)--(ao);
\draw[fl] (bi)--(b1);
\draw[fl] (b3)--(bo);
\draw[-{Stealth[length=2mm]},inkP,line width=0.8pt] (ao.east) -- node[above,font=\scriptsize\bfseries,text=inkP]{propose} (bi.west);

% refine loop
\draw[-{Stealth[length=2mm]},inkA,dashed,line width=0.8pt]
  (EB.north) .. controls +(-2,1.0) and +(2,1.0) .. (EA.north)
  node[midway,above,font=\scriptsize\itshape,text=inkA]{refine: failed candidate $\Rightarrow$ split abstraction, re-propose};

\node[text=inkG,align=left,text width=34mm,font=\scriptsize] at (0.4,-3.5)
  {grading is definitional:\\smaller cause--witness set $\Rightarrow$ higher $\rho$};
\end{tikzpicture}}
\caption{\textbf{The two-layer attribution engine.} \emph{Engine~A} works on
a discrete structural abstraction of the certified subgraph: an exact search
over the gated candidates enumerates \emph{minimal} cause sets with their
witness sets, so the responsibility grade
$\rho=1/(\lvert X^\ast\rvert+\lvert W^\ast\rvert)$ of
Eq.~\eqref{eq:rho} is produced directly by the enumeration. \emph{Engine~B}
certifies each proposal by re-execution: witnesses frozen at recorded values,
seeded models and simulator rolled forward from the $\varphi$ checkpoints,
and the risk reduction checked against the margin $\delta$, exactly on a
pinned stack, with a confidence interval and action-match score on a hosted
one. A failed proposal is not reported; it triggers abstraction refinement.
The split keeps the worst-case-hard search tractable on incident slices while
keeping the verdict sound: Engine~B never trusts Engine~A's abstraction
(Proposition~\ref{prop:sound}).}
\label{fig:engines}
\end{figure}
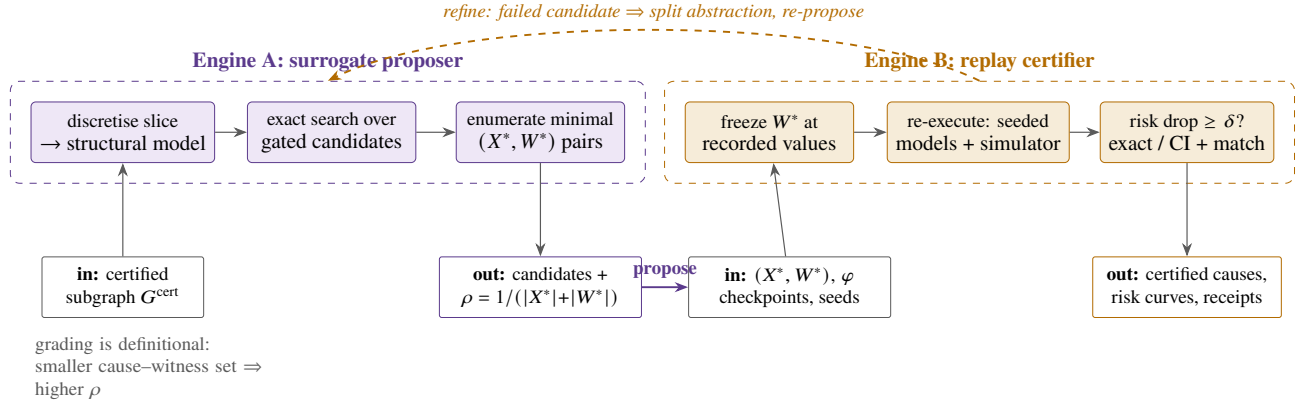

\begin{table}[H]
\centering
\caption{\textbf{CulpaBench, the causal-attribution benchmark.} Each structure carries
planted graded ground truth and probes a distinct attribution challenge.
Complete generators and the analytic oracle are in the Supplementary
Information.}
\label{tab:scenarios}
\small
\begin{tabular}{@{}lll@{}}
\toprule
Structure & Graded ground truth & Attribution challenge it probes \\
\midrule
Redundancy & $\rho=\tfrac12,\tfrac12$; two certified causes & overdetermination \\
Inert twin & one cause $\rho{=}1$; twin: breach, $\rho{=}0$ & transcript-identical evidence \\
Preemption & preempted command exonerated & effect that never arrived \\
Omission & absence variable graded; omitter blamed & responsibility for inaction \\
Delegation chain & chain graded; faithful relay flagged & originator versus conduit \\
Unexpected startup & omitted halt graded; sole actor innocent & the culprit who did nothing \\
Record attack & verdict per Proposition~\ref{prop:recordedit} & tampered evidence \\
No-cause control & universal exoneration & manufacturing a culprit \\
Trivial control & one culprit, $\rho{=}1$ & the single-culprit floor \\
\bottomrule
\end{tabular}
\end{table}

\begin{table}[H]
\centering
\caption{\textbf{Exactness holds; the cost is computational} (facility
register). Responsibility error stays $0.000$ at every tested incident width
$k$, while wall-clock per incident-slice query roughly doubles per added
candidate. Record overhead is modest: $393$\,B per signed message and
$118$\,KB per complete incident.}
\label{tab:cost}
\small
\begin{tabular}{@{}lcccc@{}}
\toprule
Interacting candidates $k$ & $2$ & $8$ & $12$ & $14$ \\
\midrule
Responsibility error $\downarrow$ & $0.000$ & $0.000$ & $0.000$ & $0.000$ \\
Wall-clock per query (s) & $0.02$ & $0.92$ & $15.0$ & $62.1$ \\
\bottomrule
\end{tabular}
\end{table}

% ===================== SUPPLEMENTARY INFORMATION =====================
\clearpage
\beginsupplement
% audita-si-body.tex — SI content, shared by the standalone SI and the combined review file.

\begin{center}
{\LARGE\bfseries Supplementary Information\par}
\end{center}
\vspace{0.8em}

\noindent This document contains the formal model and notation
(Note~\ref{snote:model}), proofs of the record-layer results together with two
supporting propositions (Note~\ref{snote:proofs}), the guarantees against
valid-key adversaries (the two attack constructions, culpability groundedness,
blame-shift accountability, the completeness barrier, and exoneration
accountability, with full proofs; Note~\ref{snote:adversarial}), the
framing attacks as implemented and what their success does not show
(Note~\ref{snote:framingimpl}), the scenario library with planted ground truth
and the analytic oracle (Note~\ref{snote:scenarios}), 
% the registered predictions and the outcomes
% the runs assigned them (Note~\ref{snote:predictions}), 
the canonical form and
granularity-invariance result (Note~\ref{snote:canonical}),
record-substrate and replay details (Note~\ref{snote:substrate}),
reproducibility information (Note~\ref{snote:repro}),  experimental deviations and validation (Note~\ref{snote:disclosures}), and an
extended survey of related work across every line of literature this project
engaged (Note~\ref{snote:extrelated}). Citations are numbered superscripts resolved
against the reference list of this document, which the supplementary notes
share with the article and which precedes the Extended Data; main-text
equation numbers are written as ``Eq.~(1) of the main text.''

\section{Notation and formal model}
\label{snote:model}

\paragraph{Record and verification.}
A message is $m=(a,\pi,C,\varphi,\sigma)$ as in Definition~1 of the main
text: author $a$, payload $\pi$, citation set $C$ of message hashes, effect
predicate $\varphi$, and signature $\sigma$ over the preceding fields. The
record $R$ is the multiset of items held by the recorder, together with
delivery receipts (co-signatures acknowledging receipt) and per-epoch Merkle
commitments over the hashes of sealed items. The verification predicate
$V(m;R)$ holds when (i) $\sigma$ verifies under $a$'s public key, (ii) every
hash in $C$ resolves to an item of $R$ satisfying $V$ (well-founded because
citations point backwards in sealed order), (iii) receipts referencing $m$
are consistent, and (iv) $m$'s hash carries a valid inclusion proof in its
epoch commitment. This induces the partition of $R$ into the certified graph
$\Gcert(R)=\{m : V(m;R)\}$ with citation edges, the \emph{suspect} set
(present, failing some check), and the \emph{missing} set
\begin{equation}
\begin{split}
M(R) \;=\; \{\, h \;:\;{}& h \text{ is cited by some } m\in\Gcert(R),
\text{ or appears in a receipt}\\
& \text{or a sealed epoch commitment, and no item of } R \text{ hashes to }
h \,\}.
\end{split}
\end{equation}

\paragraph{Slice model.}
For a declared adverse predicate $Y$ true at $t^{\ast}$, the incident slice
$S(Y)$ is the citation-ancestry of the actuation events referenced by $Y$
within $\Gcert$, closed under the duties active in the window. The slice
model $\mathcal{M}_S$ is a structural causal model whose variables are (a)
one binary presence variable per certified message in $S(Y)$, (b) one binary
\emph{absence} variable per active duty whose required message was not
issued (recorded as a first-class fact by the duty registry), and (c) the
outcome $Y$; mechanisms are induced by the citation structure and the
recorded actuation semantics. Messages excluded by the production gate
(main text, Methods) do not enter $\mathcal{M}_S$; as a consequence of this
encoding convention, preemption is resolved at the evidence layer, and the
graded model over gated candidates is \emph{monotone} in all scenario
families of this paper: every mechanism is nondecreasing in the presence
variables, with omissions represented positively by their absence variables.

\paragraph{Attribution and completions.}
On a fixed model, the modified Halpern--Pearl test\cite{halpern2015,halpernpearl2005} and the grades of
Eqs.~(1)--(2) of the main text define $\rho_p(\mathcal{M})$ for each
principal $p$. To account for imperfect records, \audita{} evaluates
$\rho_p$ over \emph{admissible completions}: assignments
$c:M(R)\to\{\text{present},\text{absent}\}$ of the missing set that are
consistent with the verified receipts, commitments, and the recorded
actuation facts (in particular the occurrence of $Y$), each
inducing a model $\mathcal{M}_c$ (a missing node participates through the
citation edges declared by the certified messages that cite it; suspect
items are treated as absent for grading and reported separately). Writing
$\mathcal{C}(R)$ for the set of admissible completions,
\begin{equation}
\label{seq:interval}
\underline{\rho}_p(R) \;=\; \min_{c\in\mathcal{C}(R)}
\rho_p(\mathcal{M}_c),
\qquad
\overline{\rho}_p(R) \;=\; \max_{c\in\mathcal{C}(R)}
\rho_p(\mathcal{M}_c),
\end{equation}
and the reported interval is $\rhointerval_p$. The lower end
$\underline{\rho}_p$ is the \emph{involvement bound}: the causal involvement that
every reading of the evidence consistent with the certified record must
concede. Involvement alone is never culpability: the verdict finds a principal
culpable only on breach and causation together
(Note~\ref{snote:adversarial}).

\section{Proofs}
\label{snote:proofs}

\subsection*{Assumptions}

\begin{assumption}[Unforgeable signatures]
\label{ass:sig}
The signature scheme is existentially unforgeable under chosen-message
attack; an adversary without $p$'s signing key cannot produce a new item
passing signature verification as $p$\cite{bernstein2012}.
\end{assumption}

\begin{assumption}[Collision-resistant hashing]
\label{ass:hash}
The hash used for citations, receipts, and Merkle commitments is
collision-resistant; an adversary cannot produce a second preimage for a
committed hash\cite{merkle1988}.
\end{assumption}

\begin{assumption}[Sealed epochs]
\label{ass:seal}
Attacks on the record occur against sealed epochs: every certified message
is covered by a published, append-only epoch commitment before the
attack\cite{merkle1988,laurie2014}. (Deletion inside an unsealed window is a compromise of the live
recorder, bounded operationally by epoch length and recorder replication;
it is outside the theorem's scope, as stated in the main-text threat model.)
\end{assumption}

\begin{assumption}[Key custody]
\label{ass:key}
A certified message is attributed to the holder of the signing key. An
insider signing with a valid key is outside the record-attack adversary
class: their items certify and are attributed to them.
\end{assumption}

\subsection*{Two lemmas}

\begin{lemma}[Verification soundness]
\label{lem:sound}
Let $R'$ be obtained from $R$ by injecting items that fail verification
(forged signatures, unresolvable citations, or invalid inclusion proofs).
Then $\Gcert(R')=\Gcert(R)$, $M(R')=M(R)$, and the suspect set grows by
exactly the injected items.
\end{lemma}

\begin{proof}
Each injected item fails $V$ by construction (Assumptions
\ref{ass:sig}--\ref{ass:seal} guarantee the adversary cannot make a new item
pass: it cannot forge a signature, cannot mint a preimage for an existing
committed hash, and cannot fabricate an inclusion proof for an unsealed
item), so no injected item enters $\Gcert$. Membership of the original items
in $\Gcert$ is unaffected: $V(m;R)$ depends on $m$'s own signature, on the
resolvability of $m$'s citations to \emph{certified} items, and on $m$'s
inclusion proof, none of which is altered by the presence of additional
non-verifying items. $M$ collects unresolved hashes referenced from the
certified side, which is unchanged. \qedhere
\end{proof}

\begin{lemma}[Completion monotonicity]
\label{lem:mono}
Let $R'$ be obtained from $R$ by deleting a set $D$ of items. Then every
admissible completion of $R$ corresponds to an admissible completion of
$R'$ inducing the same slice model; hence
$\mathcal{C}(R)\hookrightarrow\mathcal{C}(R')$ model-preservingly, and
$\mathcal{C}(R')$ may in addition contain completions with no counterpart
in $\mathcal{C}(R)$.
\end{lemma}

\begin{proof}
Consider first a deleted item $d\in\Gcert(R)$. By
Assumption~\ref{ass:seal}, $d$'s hash remains in its published epoch
commitment (and, if $d$ was cited or receipted, in those references), so
$d\in M(R')$: the deletion is detectable and $d$ becomes a missing node
rather than vanishing from the analysis. For any completion
$c\in\mathcal{C}(R)$, define $c'\in\mathcal{C}(R')$ by $c'(h)=c(h)$ on
$M(R)$ and $c'(\mathrm{hash}(d))=\text{present}$ for each $d\in
D\cap\Gcert(R)$. Since a present missing node participates through exactly
the citation edges declared by its certified citers (the edges
$d$ carried when certified), $\mathcal{M}_{c'}$ over $R'$ equals
$\mathcal{M}_{c}$ over $R$. Deleting a suspect item changes neither
$\Gcert$ nor $M$ nor any $\mathcal{M}_c$. Completions of $R'$ assigning
\emph{absent} to some deleted certified item are new members of
$\mathcal{C}(R')$ with no counterpart in $\mathcal{C}(R)$, which
establishes the (possibly strict) enlargement. \qedhere
\end{proof}

\subsection*{Proposition 3 (record-edit monotonicity)}

\begin{thmrestate}
Let $p$ be a principal all of whose messages verify, and let an adversary
modify the record by (i) injecting items that fail verification and/or (ii)
deleting items. Then the involvement bound $\underline{\rho}_p$ does not
increase. Under (i) alone, $\Gcert$, and hence the entire reported
verdict, is unchanged; under (ii), the reported interval $\rhointerval_p$
can only widen.
\end{thmrestate}

\begin{proof}
Under (i), Lemma~\ref{lem:sound} gives $\Gcert(R')=\Gcert(R)$ and
$M(R')=M(R)$, so $\mathcal{C}(R')=\mathcal{C}(R)$ and every reported
quantity, including $\rhointerval_p$, is bitwise identical. Under (ii),
$p$'s own messages remain certified ($p$ is fully certified and deletion of
\emph{other} items cannot invalidate $p$'s signatures, citations to
certified items, or inclusion proofs; deletion of one of $p$'s own items
moves it to $M$ by Lemma~\ref{lem:mono}, where completions may restore it).
By Lemma~\ref{lem:mono}, $\mathcal{C}(R)$ embeds model-preservingly into
$\mathcal{C}(R')$, so
\[
\underline{\rho}_p(R') \;=\; \min_{c'\in\mathcal{C}(R')}
\rho_p(\mathcal{M}_{c'})
\;\le\;
\min_{c\in\mathcal{C}(R)} \rho_p(\mathcal{M}_{c})
\;=\; \underline{\rho}_p(R),
\]
a minimum over a superset. The same embedding gives
$\overline{\rho}_p(R')\ge\overline{\rho}_p(R)$, i.e.\ widening. Composing
(i) and (ii) in any order composes an equality with a non-increase.
\qedhere
\end{proof}

\begin{sremark}[Why the guarantee is stated for the lower end]
\label{srem:upper}
Deletion \emph{can} raise $\overline{\rho}_p$, and should. In the redundancy
the redundancy family ($Y$ true if either of two independently sufficient commands
$a_1,a_2$ executes; ground truth $\rho=\tfrac12,\tfrac12$), deleting $a_2$
from the record leaves a completion in which $a_2$ is absent; in that
reading $a_1$ is the sole cause with $\rho=1$. \audita{} then reports
$\rhointerval_{a_1}=[\tfrac12,1]$: the evidence no longer excludes sole
responsibility, and pretending otherwise would be dishonest. What the
proposition guarantees is that no such attack moves the \emph{involvement bound}:
the certified evidence still forces only $\tfrac12$, the widened upper end
is explicitly labelled as uncertainty created by missing evidence, and a
consumer applying an in-dubio-pro-reo standard is unaffected. This is the
formal sense in which the design privileges the presumption of innocence.
\end{sremark}

\subsection*{Proposition 1 (soundness by certification)}

\begin{proprestateA}
Every cause reported by \audita{} is replay-verified: intervening on
$X^{\ast}$ with $W^{\ast}$ frozen at recorded values reduces the certified
risk of $Y$ by at least $\delta$: exactly on a pinned stack, and at
confidence $1-\alpha$ with reported action-match on a hosted stack.
Certified false positives are excluded on pinned stacks and occur with
probability at most $\alpha$ per reported cause on hosted ones.
\end{proprestateA}

\begin{proof}
The certifier is a filter on Engine~A's proposals: a proposal is reported
only if its replay check passes, so the claim reduces to the semantics of
the check. \emph{Pinned stack.} All components are deterministic given
recorded seeds and versions \emph{under the serial replay protocol} (one
replay at a time on an idle server; concurrency perturbs reduction order and
is excluded from the pinned definition), so the factual replay reproduces the recorded
trajectory exactly (a mismatch aborts certification and is itself a
reportable integrity finding), and the interventional replay
$do(X^{\ast}{\leftarrow}x',\,W^{\ast}{\leftarrow}\text{recorded})$ computes
the counterfactual outcome exactly; the check passes iff the risk of $Y$
drops by at least $\delta$, which is then a verified fact, not an estimate.
\emph{Hosted stack.} Let $q_0$ and $q_1$ be the true probabilities of $Y$
under the factual and interventional distributions induced by the
non-deterministic components, and let $\hat q_0,\hat q_1$ be empirical
frequencies over $n$ independent replays each. The certifier reports only
if $\hat q_0-\hat q_1\ge\delta+2\varepsilon(n,\alpha)$ with
$\varepsilon(n,\alpha)=\sqrt{\ln(4/\alpha)/(2n)}$. By Hoeffding's
inequality\cite{hoeffding1963}, $\Pr[\,|\hat q_i-q_i|\ge\varepsilon\,]\le\alpha/2$ for
each $i\in\{0,1\}$, so by a union bound, with probability at least
$1-\alpha$ both estimates are $\varepsilon$-accurate, and on that event
$q_0-q_1\ge(\hat q_0-\hat q_1)-2\varepsilon\ge\delta$ whenever the report
fires. A cause whose true reduction is below $\delta$ is therefore reported
with probability at most $\alpha$. The
action-match score (Note~\ref{snote:substrate}) is reported alongside to
expose distributional drift between the recorded and replayed trajectories,
which the probabilistic guarantee alone does not surface. In both stacks
the engine's failure mode under check failure is abstention and
refinement, never assertion. \qedhere
\end{proof}

\subsection*{Proposition 2 (preemption retirement)}

\begin{proprestateB}
A message with no intact production path to $Y$ in $\Gcert$ receives
$\rho=0$ and an exoneration note, regardless of its content.
\end{proprestateB}

\begin{proof}
The gate defines the candidate universe: a production path for $m$ is a
citation path $m\to m_1\to\cdots\to m_k\to \mathrm{act}(Y)$ in $\Gcert$
whose every link is verified and whose every committed effect
$\varphi_{m_i}$ is confirmed by the recorded state deltas; $m$ enters the
slice model only if such a path exists. Cause sets are subsets of the slice
model's variables, so a gated-out $m$ belongs to no $X^{\ast}$, the maximum
in Eq.~(2) of the main text ranges over no set containing $m$, and
$\rho$-mass attributable to $m$ is zero; if the principal has no gated
message at all, Eq.~(2) is a maximum over the empty set, defined as $0$ and
emitted with an exoneration note. When a path is broken only by a member of
the missing set, intactness is evaluated per completion, and the candidate
contributes to $\overline{\rho}$ in completions restoring the link but not
to the involvement bound, consistent with Proposition~3 of the main text. \qedhere
\end{proof}

\subsection*{Correspondence with the NESS test on monotone models}

The modified Halpern--Pearl definition\cite{halpern2015} evaluates a pair
$(X^{\ast},W^{\ast})$: with $W^{\ast}$ held at its recorded values, setting
$X^{\ast}$ to counterfactual values must falsify $Y$, with $X^{\ast}$
minimal. The NESS test from tort doctrine\cite{wright1985} declares $x$ a cause when
$x$ is a \emph{necessary element of a set of actual conditions sufficient}
for the outcome. On the monotone slice models of
Note~\ref{snote:model} the two agree at the level of causal membership, and
the witness machinery becomes vacuous:

\begin{lemma}[Witness vacuity on monotone models]
\label{lem:witness}
In a monotone slice model, if $(X^{\ast},W^{\ast})$ satisfies the modified
test, then so does $(X^{\ast},\varnothing)$. Consequently every minimal
cause--witness pair has $W^{\ast}=\varnothing$ and grade
$\rho=1/\lvert X^{\ast}\rvert$.
\end{lemma}

\begin{proof}
Consider the intervention $X^{\ast}\!\leftarrow\!0$ (falsifying
counterfactual values in a monotone model set the relevant presences to
absent). Propagating through monotone mechanisms, every non-intervened
variable's value is weakly below its recorded value, since inputs only
decreased. Freezing $W^{\ast}$ at recorded values holds those variables
weakly \emph{above} their propagated values, and by monotonicity of every
mechanism downstream, $Y$ with the freeze is weakly above $Y$ without it.
The modified test with witness requires $Y=0$ under the freeze; therefore
$Y=0$ without it, which is the test for $(X^{\ast},\varnothing)$.
Minimality then forces the empty witness (any nonempty $W^{\ast}$ only
enlarges $\lvert X^{\ast}\rvert+\lvert W^{\ast}\rvert$), and Eq.~(1) of the
main text reduces to $1/\lvert X^{\ast}\rvert$. This also exhibits the
consistency of Eq.~(1) with the classical degree of responsibility
$1/(k{+}1)$\cite{chocklerhalpern2004}: the size-$(k{+}1)$ minimal causes of the modified
definition correspond to singleton causes with size-$k$ contingencies under
the original definition\cite{halpern2015,halpernpearl2005}.
\qedhere
\end{proof}

\begin{sproposition}[NESS correspondence]
\label{sprop:ness}
Let $\mathcal{M}$ be a monotone slice model with $Y$ true at the recorded
assignment, and let $T_1,\dots,T_r$ be the satisfied prime implicants of
$Y$ (the minimal sets of recorded-true variables sufficient for $Y$;
an antichain). Then:
(a) the minimal causes of the modified test are exactly the minimal
transversals of the hypergraph $\{T_1,\dots,T_r\}$; and
(b) a variable belongs to some minimal cause \emph{iff} it is a NESS
cause, i.e.\ belongs to some $T_i$.
\end{sproposition}

\begin{proof}
(a) By Lemma~\ref{lem:witness} it suffices to consider empty witnesses.
Setting $X^{\ast}\!\leftarrow\!0$ falsifies $Y$ iff no satisfied prime
implicant survives, i.e.\ iff $X^{\ast}\cap T_i\ne\varnothing$ for every
$i$: $X^{\ast}$ is a transversal. Minimal causes are therefore exactly
minimal transversals.

(b) ($\Rightarrow$) If $v$ lies in a minimal transversal $X^{\ast}$, then
by minimality there is an edge $T_i$ with $X^{\ast}\cap T_i=\{v\}$
(otherwise $X^{\ast}\setminus\{v\}$ would still be a transversal), so
$v\in T_i$ and $v$ is a NESS cause: $T_i$ is a set of actual conditions
sufficient for $Y$, and $T_i\setminus\{v\}$ is insufficient by primality.
($\Leftarrow$) Let $v\in T_j$ for some $j$. Because the prime implicants
form an antichain, no other edge is contained in $T_j$, so every
$T_i\ (i\ne j)$ contains a vertex $w_i\notin T_j$. The set
$X_0=\{v\}\cup\{w_i : i\ne j\}$ is a transversal, and
$X_0\cap T_j=\{v\}$ since each $w_i\notin T_j$. Any transversal subset of
$X_0$ must therefore contain $v$; pruning $X_0$ to a minimal transversal
preserves $v$. Hence $v$ belongs to a minimal transversal, i.e.\ to a
minimal cause by (a). \qedhere
\end{proof}

\begin{sremark}
The correspondence is deliberately stated at the membership level: NESS is
a binary test, while Eqs.~(1)--(2) of the main text refine it with a degree
($1/\lvert X^{\ast}\rvert$ on monotone models) and, on non-monotone models
arising outside the gate convention, with witness sets. The practical
consequence claimed in the main text is exactly this: the quantity
\audita{} computes specialises, on the model class where legal intuition is
sharpest, to a test the legal literature already recognises.
\end{sremark}

% audita-snote-adversarial.tex — Supplementary Note: guarantees against valid-key adversaries.
% Drafted 2026-07-29 (WS1, Posture A); integrated into the supplementary the same day.
% Source of statements: research/2026-07-29-theory-upgrade-T2-T3-draft.md (v3).

\section{Guarantees against valid-key adversaries: groundedness, blame-shift
accountability, and the completeness barrier}
\label{snote:adversarial}

Record-edit monotonicity (Proposition~3 of the main text, proven in
Note~\ref{snote:proofs}) concerns an adversary who attacks the \emph{record}: it injects items
that fail verification or deletes sealed items, and the guarantee follows from verification
soundness and completion monotonicity. This note upgrades the adversary to the
strongest one admitted by the main-text threat model: a coalition holding \emph{valid keys},
acting \emph{during} the incident. Such a coalition needs no forgery: it can attack through
conduct. We first exhibit two in-model attacks showing that raw causal responsibility is
manufacturable by conduct (and that two natural localisation claims are false); we then prove
what the composed verdict still guarantees.

\paragraph{Adversary model.}
A \emph{valid-key coalition} is a set $A \subseteq P$ of principals that may
(a)~choose their live policies adaptively, including policies that condition on other
principals' messages; (b)~author arbitrary payload content, including false assertions about
other principals; and (c)~apply, after the incident, any record modification of the record-edit
class (injection of non-verifying items; deletion against sealed epochs;
Proposition~3, main text). Assumptions
A1--A4 of Note~\ref{snote:proofs} are in force; in particular, by key custody (A4), every coalition message
certifies normally and is attributed to its coalition author.

\subsection*{Two attacks within the model}

\begin{sremark}[Attack I: conditioning]
\label{srem:conditioning}
Let $v$ be an honest principal and $q \in A$ run the live policy: \emph{if a message authored
by $v$ appears in my input frontier, emit the harmful command $s_q$, citing $v$'s message;
otherwise behave normally.} In the realised run $v$ emits its ordinary compliant message
$m_v$; $q$ emits $s_q$ citing $m_v$; the adverse outcome $Y$ occurs. Every item certifies, so
the missing and suspect sets are empty, the completion set is a singleton, and
$\underline{\rho}_v = \rho_v$. The production gate passes for $m_v$: the path $m_v \to s_q \to
\mathrm{act}(Y)$ exists in $\Gcert$ \emph{by the attacker's own citation}, which is an authored
claim about $q$'s conduct and therefore not excludable as hearsay. The modified
Halpern--Pearl test with $X^{*}=\{m_v\}$, $W^{*}=\emptyset$ succeeds: intervening
$do(m_v{\leftarrow}\text{absent})$ sends $q$'s policy down its normal branch and $Y$ is
falsified. Hence $\rho_v = 1$, and the replay \emph{honestly certifies} it: the
counterfactual is true. Raw responsibility is thus manufacturable by conduct alone, with no
record attack; record-edit monotonicity is not violated but silent. Note that here the attacking node lies on
the production path and is itself an actual cause of $Y$, properties that
Remark~\ref{srem:rigging} shows are \emph{not} general.
\end{sremark}

\begin{sremark}[Attack II: inevitability rigging]
\label{srem:rigging}
Binary slice model: victim message $m_v$ (actual value $1$), a coalition-authored
configuration $c$ with realised (deviated) value $\mathit{dev}$ and reference value
$\mathit{ref}$ (Definition~\ref{sdef:refstd}), outcome mechanism $Y = f(m_v, c)$ with
\begin{center}
\begin{tabular}{@{}ccc@{}}
\toprule
$m_v$ & $c$ & $Y = f(m_v,c)$\\
\midrule
$1$ & $\mathit{dev}$ & $1$ \quad (realised: adverse)\\
$0$ & $\mathit{dev}$ & $0$\\
$1$ & $\mathit{ref}$ & $1$\\
$0$ & $\mathit{ref}$ & $1$\\
\bottomrule
\end{tabular}
\end{center}
Under the reference value the harm is \emph{inevitable} ($Y=1$ regardless of $m_v$), so
$\rho_v = 0$ there. Under the deviation, $do(m_v{\leftarrow}0)$ falsifies $Y$, so $\rho_v = 1$
with $W^{*}=\emptyset$: the deviation converts an inevitable harm into one that hinged on the
victim. Two further facts, checked directly against the modified HP definition: (i)~$c$ is
\emph{not} an actual cause of $Y$ in the realised model: $do(c{\leftarrow}\mathit{ref})$
leaves $Y=1$ (row~3), with or without freezing $m_v$, and no superset of $\{c\}$ helps; and
(ii)~$c$ lies on no production path from $m_v$ to $Y$. Consequently, both of the following
tempting strengthenings of our guarantees are \textbf{false}: ``every manufactured-involvement
path carries a coalition-authored node'' (path localisation), and ``any raise of a victim's
responsibility implies some coalition variable has positive responsibility for $Y$''
(object-level localisation). The correct level of analysis is \emph{meta}-causal
(Theorem~\ref{sthm:blameshift}): $c$ \emph{is} an actual cause of $v$'s pivotality. Writing
$P_v \equiv [\,Y \wedge (do(m_v{\leftarrow}0)\Rightarrow\neg Y)\,]$, we have $P_v = 1$ in the
realised model while $do(c{\leftarrow}\mathit{ref})$ falsifies it (row~4).
\end{sremark}

\begin{sremark}[Lineage]
An earlier design iteration repaired \emph{record-level} framing: adversarial
completions could raise an honest principal's computed $\rho$, which the restriction of
verdict evidence to $\Gcert$ eliminated (the setting of Proposition~3, main text).
Remarks~\ref{srem:conditioning}--\ref{srem:rigging} are the \emph{conduct-level} analogue,
which no record restriction can repair; Theorems~\ref{sthm:grounded} and
\ref{sthm:blameshift} are its fix.
\end{sremark}

\subsection*{Content-determined duties and groundedness}

\begin{sdefinition}[Content-determined duty roster]
\label{sdef:cdd}
A duty roster $\{\Phi_p\}_{p \in P}$ is \emph{content-determined} if each $\Phi_p$ is a
computable predicate over exactly (i)~the certified messages authored by $p$, together with
$p$'s duty-registered absence variables, and (ii)~the certified input frontier of each such
item (the certified items it cites, and the certified items deliverable to $p$ at issue time
per the recorded receipts). In particular $\Phi_p$ takes no ground-truth label, no post-hoc
information, and no third-party payload \emph{about} $p$ except through channel~(ii), where
such a payload bears only on the compliance of $p$'s \emph{response} to it.
Concrete instances used in this paper: the aggregator duty (output equals the declared
normative aggregation of its certified inputs) and the solver self-consistency duty (the
emitted final answer equals the answer derived in the solver's own certified working).
\end{sdefinition}

\begin{sdefinition}[Culpability finding]
\label{sdef:culp}
The verdict \emph{finds $p$ culpable} iff it reports a breach finding for $p$ and
$\underline{\rho}_p > 0$. Following the evidence-chain rule (verdicts may cite only evidence inside
$\Gcert$; Note~\ref{snote:model}), a breach finding
must cite a witness inside $\Gcert$: certified $p$-authored items (or $p$'s duty-registered
absences) on whose content $\Phi_p$ fails.
\end{sdefinition}

\begin{lemma}[Breach locality and edit-invariance]
\label{slem:locality}
Under a content-determined roster, for a principal $p$ all of whose messages verify:
(a)~the breach finding for $p$ is computable from $\Gcert$ alone; (b)~no record modification
of the record-edit class creates a breach finding for $p$: injections leave the finding unchanged,
and deletions can only destroy its witness (moving the finding to
``non-evaluable~/~evidence missing''), never create one.
\end{lemma}

\begin{proof}
(a) is Definition~\ref{sdef:cdd}. For (b): injected items fail verification and never enter
$\Gcert$ (Lemma~S1), so the arguments of $\Phi_p$ are unchanged. A deletion moves items to the
missing set (Lemma~S2). By Definition~\ref{sdef:culp} a breach finding must cite a certified
witness; the witness set after deletion is a subset of the witness set before, so the set of
assertable breach findings shrinks monotonically. No new $p$-authored certified content
appears under either operation, so no new witness, and hence no new finding, can arise.
\end{proof}

\begin{sproposition}[Hearsay-freeness of the verdict pipeline]
\label{sprop:hearsay}
For $q \neq p$, a $q$-authored certified item enters the computation of $p$'s verdict through
exactly two channels: as a variable of the slice model attributed to $q$ (mechanism channel:
its presence and content may causally influence outcomes, and interventions on it are graded
against $q$), or as an element of $p$'s certified input frontier (context channel: it bears on
whether $p$'s response was compliant). In particular, if two certified records agree on $p$'s
authored items and on their input frontiers, they yield the same breach finding for $p$,
whatever any other principal asserts; and the causal prong for $p$ is computed by
interventions on $p$'s variables only. There is no testimony channel: no assertion by $q$ is
ever evaluated as evidence of $p$'s conduct.
\end{sproposition}

\begin{proof}
By inspection of the pipeline's dependency structure, which the definitions fix: the breach
prong is $\Phi_p$, whose argument list (Definition~\ref{sdef:cdd}) contains no $q$-authored
item outside $p$'s frontier; the causal prong applies the gate--grade--certify test to
candidate sets of $p$-authored variables in $\mathcal{M}_S$, in which every $q$-authored item
occurs solely as a variable with its own mechanism, attributed to $q$ (Definition~1, main
text). The two channels are exhaustive because $\mathcal{M}_S$ contains no other occurrence of
$q$-authored content.
\end{proof}

\begin{stheorem}[Culpability groundedness]
\label{sthm:grounded}
Fix any content-determined duty roster. For every valid-key coalition $A$, every principal
$p$, and every coalition-induced run with any record modification of the record-edit
class applied: if
the verdict finds $p$ culpable, then the breach finding is witnessed by $p$'s own certified
conduct: certified $p$-authored items (or duty-registered absences of $p$'s required
items) whose content violates $\Phi_p$ on $p$'s certified input frontier. Equivalently: a
principal whose own certified conduct satisfied $\Phi_p$ in the realised run is never found
culpable; the strongest verdict against such a principal is \emph{causal involvement with
exoneration for breach}.
\end{stheorem}

\begin{proof}
By Definition~\ref{sdef:culp} a culpability finding contains a breach finding citing a
certified witness; by Definition~\ref{sdef:cdd} and
Proposition~\ref{sprop:hearsay} that witness consists of $p$-authored certified content (or
$p$'s duty-registered absences) evaluated on $p$'s certified frontier, and no coalition
conduct, payload assertion, or record edit can substitute for it
(Lemma~\ref{slem:locality}(b)). The contrapositive is immediate.
\end{proof}

\begin{scorollary}[Ex-ante safety]
\label{scor:exante}
Call $p$ \emph{robustly duty-compliant} if its policy satisfies $\Phi_p$ on every admissible
input frontier. A robustly duty-compliant principal is never found culpable under any
valid-key coalition strategy. \qed
\end{scorollary}

\begin{sremark}[What this theorem is and is not]
\label{srem:honest}
Theorem~\ref{sthm:grounded} is sound \emph{by careful construction}: it verifies that the
composed pipeline realises a dependency restriction, and we present it as such. Its content
lies in two external facts. \emph{Necessity:} Remark~\ref{srem:conditioning} shows the causal
prong alone violates groundedness (responsibility without breach is manufacturable), so
the breach$\wedge$causation composition is what carries the guarantee.
\emph{Tightness:} Theorem~\ref{sthm:barrier} shows that strengthening the breach prong to
close its completeness gap is impossible without ground truth; groundedness sits at the exact
boundary of what a record-level auditor can certify.
\end{sremark}

\subsection*{Blame-shift accountability}

\begin{sdefinition}[Declared reference standard; deviation set]
\label{sdef:refstd}
A \emph{reference standard} for principal $a$ is a declared policy $\pi^{\mathrm{ref}}_a$
(for example, the vendor's own agent under clean prompt, or a certified checker) registered
in the duty registry before the incident, or declared by the investigator with disclosure.
Given a realised run on a pinned stack, the \emph{deviation set} $D$ of a coalition $A$ is the
set of authored-variable families at which the realised conduct differs from
$\pi^{\mathrm{ref}}$'s conduct on the same certified input frontier; it is computable by
pinned replay of $\pi^{\mathrm{ref}}$ on that frontier. All results below are parametric in
the declared standard: different standards yield different certificates, each valid relative
to its declaration.
\end{sdefinition}

\begin{sdefinition}[Deviation-indexed extension; pivotality]
\label{sdef:ext}
Let $M$ be the realised slice model and $D$ a deviation set. The extension $\widehat{M}[D]$
adds one binary selector $\sigma_i$ per deviation ($\sigma_i = 1$: realised mechanism;
$\sigma_i = 0$: reference mechanism); its variable universe is the union of the potential
messages of the $2^{|D|}$ selector settings (a message not issued under a setting is absent),
which is finite on a pinned stack since each setting determines one finite run; all other
mechanisms are unchanged; the actual context is $\sigma \equiv 1$. For a fully certified
principal $v$ and level $r$, the \emph{pivotality event} is
$P^r_v \equiv [\,\underline{\rho}_v \geq r\,]$, an event of $\widehat{M}[D]$ whose value under
a setting $\sigma$ is computed on the run that $\sigma$ induces.
\end{sdefinition}

\begin{stheorem}[Blame-shift accountability]
\label{sthm:blameshift}
Let $v \notin A$ be fully certified and duty-compliant in the realised run, and suppose the
coalition's conduct deviations raised $v$'s involvement bound:
$\underline{\rho}_v(M) = r > \underline{\rho}_v(M^{\mathrm{ref}})$, where
$M^{\mathrm{ref}} = \widehat{M}[D]|_{\sigma \equiv 0}$ and $\underline{\rho}_v := 0$ when the
induced run is not adverse. Then:
\begin{enumerate}
\item \emph{(Existence.)} There is a nonempty minimal $X \subseteq D$ and a witness
$W \subseteq D \setminus X$, frozen at deviated values, such that
$do(\sigma_X{\leftarrow}0,\ \sigma_W{\leftarrow}1)$ falsifies $P^r_v$: under the modified
Halpern--Pearl definition, \textbf{the coalition's deviations are an actual cause of the
victim's pivotality}.
\item \emph{(Certificate.)} Any such minimal pair $(X, W)$ is a machine-checkable
\emph{blame-shift certificate}: it names concrete deviations whose restoration to the declared
standard destroys the raise, and it is verifiable by replaying the two runs it distinguishes.
\item \emph{(Graded framing responsibility.)} Each $a \in A$ receives
$\rho^{\mathrm{frame}}_a = \max\{\, 1/(|X|+|W|) \,:\, (X,W)$ minimal for  $P^r_v,\ X
\text{ contains a deviation authored by } a \,\}$: the framers are graded by the same
responsibility calculus they abused.
\item \emph{(No stronger localisation.)} Clause~1 cannot be strengthened to place a deviation
on a production path from $v$'s messages to $Y$, nor to make some deviation an actual cause of
$Y$ itself: Remarks~\ref{srem:conditioning} and \ref{srem:rigging} realise the raise with and
without those properties.
\end{enumerate}
\end{stheorem}

\begin{proof}
Clause~1. With all non-selector exogenous conditions fixed, $P^r_v$ is a function of
$\sigma \in \{0,1\}^{D}$. By hypothesis $P^r_v(\sigma{\equiv}1) = 1$ and
$P^r_v(\sigma{\equiv}0) = 0$. Hence the pair $X = D$, $W = \emptyset$ satisfies the
modified-HP falsification condition at the actual context $\sigma \equiv 1$; the family of
satisfying sets is nonempty and finite, so it contains a minimal element $X$, and
$X \neq \emptyset$ since $P^r_v(\sigma{\equiv}1) = 1$. Freezing sets arising in minimisation
are frozen at their actual (that is, deviated) values, as the modified definition
requires. Clause~2 restates clause~1's witness operationally; verifiability is pinned-stack
replay (Proposition~1, main text). Clause~3 is Chockler--Halpern applied in
$\widehat{M}[D]$. Clause~4 is by the two exhibited attacks.
\end{proof}

\begin{scorollary}[Composition with record attacks]
\label{scor:compose}
Against a coalition using conduct deviations and record-edit-class attacks together, the
guarantees compose: the edits do not raise $\underline{\rho}_v$ beyond its pre-edit value
(Proposition~3, main text), any conduct-driven raise above the declared-reference value is certificated by
Theorem~\ref{sthm:blameshift} on the pre-edit record, and no combination yields a culpability
finding against a compliant $v$ (Theorem~\ref{sthm:grounded}). \qed
\end{scorollary}

\begin{sremark}[Verdict integration]
When the engine reports involvement-without-breach for a compliant principal
(Theorem~\ref{sthm:grounded}) and a reference standard is declared, it may additionally emit
the blame-shift certificate and the $\rho^{\mathrm{frame}}$ grades: ``you can be framed into
the causal story'' becomes ``\ldots and the framing itself is attributed and graded.''
Computing certificates is exhaustive search over deviation subsets, the same regime as the
main engine; the tractability boundary is measured, not asserted, as elsewhere in the paper.
\end{sremark}

\subsection*{The completeness barrier}

Throughout this subsection fix a task family with gold labels $g$, and call certified conduct
of a principal \emph{outcome-culpable} (relative to gold) if its emitted final answer differs
from $g$ on the task instance and the conduct receives a positive certified causal grade for
the adverse outcome. This is the field's culprit notion (a wrong answer that caused the
failure); note that it is \emph{gold-referenced} by definition.

\begin{sdefinition}[Soundness and completeness of a breach standard]
\label{sdef:soundcomplete}
A content-determined $\Phi$ is \emph{sound} if it passes every conduct producible by a
duty-compliant policy on its realised certified frontier; it is \emph{complete} (relative to
gold) if it fires on every outcome-culpable conduct, over all well-formed certified records.
\end{sdefinition}

\begin{sdefinition}[Honest fallibility; template-uniform realizability]
\label{sdef:hf}
A domain exhibits \emph{honest fallibility} if some duty-compliant policy, on some certified
frontier, produces outcome-culpable conduct (a compliant honest error that causes harm). It
exhibits \emph{template-uniform realizability} if there is a constructor $K$ mapping any task
instance $t$ and candidate answer $a$ to a well-formed certified record $K(t,a)$ (a
single-solver incident whose aggregation passes the answer through, so that the conduct is
outcome-culpable iff $a \neq g(t)$) whose focal conduct is producible by a duty-compliant
policy for every $(t,a)$ in a dense subset of instances. Free-form LLM solvers satisfy both in
our corpora: honest errors exist (measured), and an internally consistent derivation ending in
an arbitrary candidate answer is compliant-producible conduct.
\end{sdefinition}

\begin{stheorem}[Completeness barrier]
\label{sthm:barrier}
\begin{enumerate}
\item Under honest fallibility, no content-determined breach standard is both sound and
complete: the requirements contradict on the honest-error conduct.
\item Under template-uniform realizability, any sound content-determined $\Phi$ that is
complete decides answer correctness: $\Phi(K(t,a))$ fires iff $a \neq g(t)$, so a single
record evaluation answers ``is $a$ the correct answer to $t$?''. Hence breach-completeness is
at least as hard as answer verification for the task family, and any auditor restricted to the
certified record, human or algorithmic, inherits the same barrier.
\end{enumerate}
\end{stheorem}

\begin{proof}
(1) Let $C^{*}$ be the conduct witnessing honest fallibility. Soundness requires $\Phi$ to
pass $C^{*}$ (it is compliant-produced); completeness requires $\Phi$ to fire on $C^{*}$ (it
is outcome-culpable). Contradiction. (2) For $a \neq g(t)$: the focal conduct of $K(t,a)$ is
outcome-culpable by construction, so completeness forces $\Phi$ to fire. For $a = g(t)$: the
conduct is compliant-producible (realizability) and not outcome-culpable, so soundness forces
$\Phi$ to pass. Thus $\Phi(K(t,a)) = [\,a \neq g(t)\,]$ on the dense subset, and $\Phi$ is
computable from the record, which contains $(t,a)$ and no gold label.
\end{proof}

\begin{sremark}[Empirical shadow; semantics of the culprit metric]
\label{srem:shadow}
The barrier retro-dicts the measured structure of the accuracy experiments: a gold-free
(deployable) arm cannot match a gold-informed judge on culprit identification (it would
otherwise decide answer correctness), so parity of the gold-informed arms is the ceiling,
and the deployable arm's misses must concentrate on honest-error and non-monotonic cases,
which the diagnosis confirms independently. The deeper reading: the field's culprit-ID metric
conflates outcome-culpability with breach-culpability, and Theorem~\ref{sthm:barrier} shows
the conflation is unrecoverable from certified records alone. Either one adopts negligence
semantics, in which case declining to convict honest errors is correct behaviour rather
than an accuracy deficit, or one demands gold at audit time, and the judge baseline
becomes an oracle rather than a deployable auditor.
\end{sremark}

\subsection*{Exoneration accountability}

\begin{assumption}[Attributed commitments]
\label{ass:receipts}
Epoch commitments published at sealing time record, for every sealed item, an item identifier
that binds the author's principal identity together with the item's hash; delivery receipts
are signed by recipients over the item identifier. Consequently, for a deleted item, its
authorship (from the sealed identifier) and its \emph{certified citers} (the surviving
certified items whose citation sets name it) remain recoverable without the item's payload.
The reference implementation satisfies this: item identifiers embed the author's principal
identity, sealed blocks retain the identifier list, and Merkle leaves commit the full signed
content.
\end{assumption}

\begin{stheorem}[No silent exoneration]
\label{sthm:exon}
Let $q$ have $\underline{\rho}_q(R) = r > 0$ on the sealed record $R$ (a certified culprit),
and let $R'$ be obtained by any record modification of the record-edit class. Then:
\begin{enumerate}
\item \emph{(Detectability with attribution.)} If $\underline{\rho}_q(R') < r$, then the
missing set $M(R')$ is nonempty and meets every forcing family for level $r$ (defined in the
proof); under Assumption~\ref{ass:receipts} each such missing item is reported with its
author and its certified citers. The verdict states: the exoneration of $q$ rests on evidence
recorded-but-now-missing, and names whose evidence it was.
\item \emph{(Guilt persistence in the interval.)}
$\overline{\rho}_q(R') \geq \underline{\rho}_q(R) = r$: the modified record still admits
$q$'s responsibility at the original level; the attack converts forced guilt into flagged
uncertainty, never into certified innocence.
\item \emph{(Public verifiability.)} With the multi-keeper seal and an honest-majority keeper
set, the pre-edit commitments are available to any auditor, so clauses~1--2 are checkable by
parties who never saw $R$.
\item \emph{(Breach-erasure accountability.)} The same holds for the other prong of
culpability: any edit that destroys a breach finding against $q$ (deleting the certified
witness items, or corrupting stored copies so that they fail verification and become
suspect) moves those witnesses to the missing or suspect set, author-attributed under
Assumption~\ref{ass:receipts}, and the verdict reports \emph{breach evaluation degraded by
missing/suspect evidence authored by $q$} rather than a clean non-breach. Exoneration by
erasure is impossible silently on either prong.
\end{enumerate}
\end{stheorem}

\begin{proof}
(1) Injections do not change any reported quantity (Lemma~S1), so a strict drop requires
deletions. Call a set $F \subseteq \Gcert(R)$ a \emph{forcing family for level $r$} if every
admissible completion in which all of $F$ is present yields $\rho_q \geq r$; since
$\underline{\rho}_q(R) = r$, the set of all certified items is itself a forcing family, so
forcing families exist. Suppose some forcing family $F$ survived intact,
$F \subseteq \Gcert(R')$. Every admissible completion of $R'$ contains every certified item,
in particular all of $F$, hence yields $\rho_q \geq r$, contradicting
$\underline{\rho}_q(R') < r$. Therefore \emph{every} forcing family for level $r$ lost at
least one member; each lost member is a deleted certified item, which by Lemma~S2 moved to
$M(R')$ (deletion against a sealed epoch is detectable), and Assumption~\ref{ass:receipts}
recovers its authorship and frontier. (2) By Lemma~S2, $\mathcal{C}(R)$ embeds
model-preservingly in $\mathcal{C}(R')$, so
$\overline{\rho}_q(R') \geq \overline{\rho}_q(R) \geq \underline{\rho}_q(R) = r$.
(3) is the standard transparency-log argument and is cited, not reproved.
(4) A breach finding must cite certified witnesses (Definition~\ref{sdef:culp}); by
Lemma~\ref{slem:locality}(b) edits can only remove such witnesses, and a removed witness is
either a sealed deletion (missing, by Lemma~S2) or a verification failure (suspect), both
visible partitions; authorship attribution is Assumption~\ref{ass:receipts}.
\end{proof}

\begin{sremark}[The two-sided characterisation]
Record-edit monotonicity and Theorem~\ref{sthm:exon} together replace the informal asymmetry claim of the
Discussion with a two-sided statement: manufactured certainty is impossible for the fully
certified (Proposition~3 of the main text, strengthened to conduct adversaries by
Theorems~\ref{sthm:grounded}--\ref{sthm:blameshift}), and manufactured innocence is
impossible \emph{silently} (Theorem~\ref{sthm:exon}): destruction buys uncertainty, and the
uncertainty arrives pre-attributed.
\end{sremark}

\subsection*{Proof of Theorem~5 of the main text (grounded framing at every order)}
\emph{(i) Completeness.} The order-$k$ model is a finite boolean structural
model over deviation selectors. If the outcome predicate takes different
values at the all-realised and all-reference assignments, then the predicate
depends on the selector vector, so there is a nonempty minimal set
$X^{\ast}$ of selectors satisfying the modified Halpern--Pearl conditions
for the realised value (flip $X^{\ast}$ under a witness $W^{\ast}$ and the
predicate changes; minimality by finiteness). Every selector is, by
construction, attached to a realised message on the certified record, and
its owner is the signer of that message; under sound key custody (A3) the
attribution of every graded variable to its owner is therefore witnessed by
a signature the owner cannot repudiate.

\emph{(ii) Groundedness, by induction on the order $k$.} The base case
$k=1$ is the separation of Theorem~\ref{sthm:grounded} applied to deviating
conduct: a culpable-framing finding is defined as a positive grade on a
deviation that itself breaches the registered content-determined roster, and
breach is computable from the owner's own certified conduct alone. For the
inductive step, observe that the lifting from order $k-1$ to order $k$
preserves the two properties on which the base case rests: every endogenous
variable remains a deviation selector of signed conduct (so ownership
remains signature-attributed), and the breach predicate remains the same
registered content-determined roster evaluated on that conduct (so
culpability remains witnessed by the owner's own record). A principal with
no deviating conduct contributes no variable to any order-$k$ model and can
therefore receive no grade; a principal whose deviations are all
duty-compliant can receive involvement grades but never satisfies the
breach conjunct. \qed

\subsection*{Proof of Proposition~4 of the main text (dilution is not free)}
Let $d$ be a causally inert deviation: the outcome predicate takes the same
value at every pair of selector assignments differing only at $d$. Then $d$
cannot belong to any minimal $X^{\ast}$ (flipping it changes nothing, so
removing it from a candidate cause preserves the counterfactual clause and
contradicts minimality) and is never required in any minimal witness
$W^{\ast}$ (its value never affects the evaluation). Existence and all
grades $1/(|X^{\ast}|+|W^{\ast}|)$ are computed from minimal sets only and
are therefore unchanged. Contrapositively, a deviation whose inclusion
changes any grade must appear in some minimal set, hence the predicate
depends on it, hence it is causally active; it is realised, signed conduct
on the record and Theorem~5 applies to it. \qed

\subsection*{Attack III: the level-2 scapegoat, and why it fails}
The registered level-2 construction (EXP-N4) arranges a \emph{benign}
deviation to be pivotal for the victim's pivotality. A compliant victim
routes a unit late in the episode; the attacker launches a breaching-speed
near miss timed against the human's pass; a supervisor-role principal halts
the victim's unit, citing the attacker's hazard as its reason, then resumes
the victim's task at the duty limit, and the delay retimes the unit into
the human's return pass. The supervisor's deviation from its declared plan
is duty-compliant (a halt, then a resume at the limit), and in every kept
draw the certified supervisor policy, replayed offline over the recorded
sensor reports, would itself have demanded the halt: the deviation is not
merely standard-compliant but operationally justified. A naive reading of
the order-one certificate blames this benign deviator, and the involvement
layer indeed grades it in every kept incident. The culpability layer never
does: the finding requires the framer's own deviation to breach the
registered roster, which the halt does not and the attacker's speeding
does.

\section{Framing attacks as implemented, and what their success does not show}
\label{snote:framingimpl}
Note~\ref{snote:adversarial} states the two conduct-level attacks abstractly. Their facility
instantiation is as follows. In the \emph{conditioning} attack the attacker issues a harmful
speed command whose citation set names the victim's compliant message; because citation is an
authored causal claim (Definition~1 of the main text), that citation is what creates the
certified production path, and suppressing the victim's message consequently silences the
attacker's. In the \emph{inevitability-rigging} attack the victim issues a duty-compliant
route, a separate principal breaches the speed duty, and the attacker's realised conduct is an
innocuous command whose declared reference counterpart would have routed a second unit through
the corridor, which would have made the collision inevitable regardless of the victim. Deviating from that
reference removes the redundant cause and leaves the victim's compliant routing as the
remaining but-for link, without the attacker ever acting on the victim's path.

The certificate is computed by the responsibility calculus applied to selector variables that
switch each deviation between realised and reference conduct, with the victim's pivotality as
the outcome event; counterfactuals are decided by exact world re-execution with citation-chain
propagation, so a suppressed trigger silences the conduct that cited it.

\emph{Scope of the adversarial evaluation.} These attacks are theorem-derived rather than independently discovered and therefore test whether the implementation realizes the predicted guarantees rather than exhaustively characterizing adversarial behavior. Both attacks nevertheless raise a compliant principal’s involvement bound to 1.0, while the resulting certificates are independently checkable by replay. Adversaries outside the stated threat model remain future work.

\section{Scenario library, planted ground truth, and the analytic oracle}
\label{snote:scenarios}

\noindent\textbf{How the families are instantiated.} The abstract family
definitions below fix the causal structure each scenario must realise; the
facility register instantiates them with physical semantics, and this
paragraph records that instantiation so the two can be checked against each
other. In the facility, principals issue typed commands (\texttt{move\_to},
\texttt{set\_speed}, \texttt{halt}, \texttt{set\_param}) carrying declared
effect predicates; a deterministic kinematic world advances at $0.1$\,s
ticks; the safety predicate fires at the first tick with human--robot
separation below $0.5$\,m, with severity $\tfrac12 m v^{2}$ at that tick, and
the quality predicate at a batch parameter outside its acceptance range.
\emph{Redundancy} is two independently sufficient speed-up commands
from distinct principals; the \emph{inert twin} is the same record with
the second command's actuator locked out, so the two families' transcripts
are bitwise identical and differ only in whether an effect propagated;
\emph{preemption} is an accused command physically halted before its
effect reaches the corridor, with the incident produced by a later command;
\emph{omission} is a registered halt duty left unfulfilled;
\emph{delegation} is a reckless route relayed by faithful conduits.

\paragraph{The physics oracle and its independence.}
Ground truth on the facility register is computed by exact world re-execution: for a candidate
set the world is re-run from its initial state with those command messages treated as absent,
and the declared predicate is evaluated on the resulting trajectory. No language model
participates at any point, so this register is immune to the serving-stack nondeterminism
that the live-corpus register must manage by protocol. The responsibility profile is then the
Chockler--Halpern grade over that outcome function, computed by an analytic implementation
written against the definitions and sharing no code with the gate-and-interval engine under
test. We exploit this to run a standing cross-implementation check: on the causal layer, where
the two share nothing, ground truth and engine must agree. The check is reported in the main
text: a gate that retired \texttt{halt} commands by kind
rather than by production path, which produced systematic under-attribution on incidents where a
halt caused harm through timing.

Two semantics matter for reproducibility. Interventions act at the
consumption level: the record is fixed, and a counterfactual world treats a
set of command messages as absent, where a command is world-effective unless
it or a policy-chain ancestor is absent. The omission variable inserts the
\emph{dutiful} halt obtained by running the certified supervisor policy
offline over the recorded sensor reports, a corrected value sourced from the
registered duty specification rather than invented. Ground truth is reported in
two layers: the causal profile (Chockler--Halpern responsibility over the
physics) and the culpable profile (causal responsibility restricted to
principals in breach of a content-determined duty). The layers differ
routinely and instructively: a compliant route-issuer is often a genuine
but-for enabler of a collision, and appears with positive causal
responsibility and zero culpability: involvement without breach, which is
the groundedness distinction of Theorem~\ref{sthm:grounded} made physical.
The \emph{sampled} register draws these dimensions independently
(redundancy width, preemption, lockout, delegation depth, omission, unit,
and the declared predicate) rather than composing a single incident type
with parameter noise; only $57.8\%$ of draws are adverse, so structure, not
construction, decides whether harm occurs.

Each family below is given by its slice-model composition; all are monotone
under the gate convention of Note~\ref{snote:model}. Ground truth is
computed by an \emph{analytic oracle} implementing
Proposition~\ref{sprop:ness} directly (prime-implicant extraction followed
by minimal-transversal enumeration on the composed outcome function), written
against the definitions and sharing no code with the engine under test.
Supplementary Figure~\ref{sfig:causal} walks the redundancy family end to end.

\paragraph{Redundancy (true overdetermination).}
Planners $a_1,a_2$ each issue an independently sufficient command; a
dispatcher $d$ cites both; the robot $r$ executes. Composed outcome
$Y=a_1\vee a_2$ through conduits $d,r$. Ground truth:
$\rho_{a_1}=\rho_{a_2}=\tfrac12$ (minimal cause $\{a_1,a_2\}$);
$\rho_d=\rho_r=1$ as singleton but-for causes, \emph{conduit-flagged};
no breach findings.

\paragraph{Inert twin.}
Transcript-identical to redundancy, but the dispatcher's command cites only
$a_1$; $a_2$'s proposal has no production path (and $a_2$ skipped a
required review step). Ground truth: $\rho_{a_1}=1$;
$\rho_{a_2}=0$ with an exoneration note \emph{and} a breach finding; the
family exists to force the breach-without-causation output. Separating redundancy
from the inert twin requires the citation structure; transcripts are
indistinguishable by construction.

\paragraph{Preemption.}
$a_1$ issues a hazardous command; a supervisor halt arrives first and the
dispatcher discards $a_1$'s command (no realised effect); a later command
$a_3$ produces $Y$. Ground truth: $\rho_{a_3}=1$; $a_1$ exonerated for $Y$
(gated out, Proposition~2) with a breach finding for the hazardous
issuance.

\paragraph{Omission.}
A sensor records human entry; the duty registry obliges the supervisor $s$
to issue a halt within $\tau$; none is issued (absence variable
$\bar H_s=1$); the standing motion command $c$ proceeds and
$Y=c\wedge\bar H_s$. Ground truth: $\rho_{\bar H_s}=1$ and $\rho_c=1$
(each conjunct is a singleton but-for cause), breach finding for $s$ only.
The family makes explicit that causation and culpability separate: the
duty-compliant issuer of $c$ is causally implicated and cleared on breach,
which the verdict states rather than hides.

\paragraph{Delegation chain.}
$u\to p_1\to p_2\to d\to$ actuation, with the erroneous transformation
introduced at $p_1$ and relayed faithfully thereafter. Ground truth: every
chain member is a singleton but-for cause ($\rho=1$); $p_2,d$ are
conduit-flagged by the faithful-relay test; the breach finding localises to
$p_1$. The family exercises the design position that $\rho$ is an input to
judgment: but-for chains grade everyone, and the verdict's flags and breach
findings carry the discrimination.

\paragraph{Record attack.}
Any of the above with an attack overlay: forge (inject items failing
verification), drop (delete sealed items), reorder, equivocate. Ground
truth: the clean-family verdict transformed exactly as Proposition~3 of the
main text prescribes: bitwise unchanged under forgery, interval-widened under
deletion, with $\underline{\rho}$ of fully-certified principals never
raised.

\paragraph{Negative controls.}
\textbf{NoCause}: $Y$ is produced by an exogenous environmental event with
no principal in the slice; correct output is universal exoneration.
\textbf{TrivialCause}: a single command with a direct path;
correct output is that command's author at $\rho=1$.

\paragraph{Parameterised generator.}
Beyond the named families, incidents are sampled from four monotone
gadgets: disjunction of width $w$ (ground truth $1/w$ each), conjunction
of width $k$ (each member $1$), chains of depth $d$ with a planted
deviation, and duty-conditioned absences. These compose into a random
series-parallel outcome function, which is then realised as a message
topology. The
oracle computes ground truth on the composed function; because gadget
composition preserves monotonicity, Proposition~\ref{sprop:ness} applies
throughout, and the sampled register measures the estimator against an
implementation-independent target.

% FIGURE 3 : causal machinery on the intra-ensemble redundancy micro-incident
\begin{figure}[t]
\centering
\resizebox{\textwidth}{!}{%
\begin{tikzpicture}[font=\scriptsize,
  msg/.style={draw=inkG,fill=white,circle,minimum size=7mm,align=center,inner sep=0.5pt},
  hum/.style={draw=inkG,fill=fillG,rounded corners=1pt,minimum width=10mm,minimum height=6mm,align=center},
  gated/.style={draw=inkG,fill=white,circle,minimum size=7mm,align=center,inner sep=0.5pt,dashed,text=inkG},
  live/.style={draw=inkP,fill=fillP,circle,minimum size=7mm,align=center,inner sep=0.5pt,line width=0.8pt},
  frozen/.style={draw=inkT,fill=fillT,circle,minimum size=7mm,align=center,inner sep=0.5pt,line width=0.8pt},
  intv/.style={draw=inkA,fill=fillA,circle,minimum size=7mm,align=center,inner sep=0.5pt,line width=0.9pt},
  outc/.style={draw=inkA,fill=white,diamond,aspect=1.4,minimum size=7mm,align=center,inner sep=0.5pt},
  ed/.style={-{Stealth[length=1.6mm]},inkG},
  pnl/.style={font=\footnotesize\bfseries}]

% ============ PANEL (a) : production gate on the intra-ensemble topology ============
\begin{scope}
\node[pnl] at (1.9,2.05) {(a) production gate};
\node[hum] (H)  at (-0.1,0.35)   {human\\lead};
\node[msg] (A1) at (1.15,1.05) {$a_1$\\pln};
\node[msg] (A2) at (1.15,-0.35){$a_2$\\pln};
\node[msg] (D)  at (2.15,0.35) {$d$\\disp};
\node[msg] (R)  at (3.1,0.35) {$v_R$\\rob};
\node[outc] (Y) at (4.05,0.35) {$Y$};
\node[gated] (S) at (2.15,-1.35){$v_S$\\sens};
\draw[ed] (H)--(A1); \draw[ed] (H)--(A2);
\draw[ed] (A1)--(D); \draw[ed] (A2)--(D); \draw[ed] (D)--(R); \draw[ed] (R)--(Y);
\draw[ed,dashed,inkG] (S) to[bend right=12] (R);
\node[draw=inkG,dashed,rounded corners=3pt,fit=(A1)(A2)(D),inner sep=1.6mm,
      label={[font=\tiny,text=inkG]above:{planning ensemble}}] {};
\begin{pgfonlayer}{background}
\draw[line width=3pt,inkP,opacity=0.18,line cap=round] (A1.center)--(D.center)--(R.center)--(Y.center);
\draw[line width=3pt,inkP,opacity=0.18,line cap=round] (A2.center)--(D.center);
\end{pgfonlayer}
\node[text=inkG,align=center,text width=54mm] at (2,-2.35)
  {$d$'s command \emph{cites} both proposals $\Rightarrow$ both live. Had it cited only $a_1$, then $a_2$ = preempted backup, gated to $\rho{=}0$ (breach still reported). $v_S$: no path, gated.};
\end{scope}

% ============ PANEL (b) : intervene a1 under frozen a2 ============
\begin{scope}[xshift=64mm]
\node[pnl] at (1.9,2.05) {(b) test $a_1$, freeze $W{=}\{a_2\}$};
\node[hum] (H2)  at (-0.1,0.35)   {human\\lead};
\node[intv] (A1b) at (1.15,1.05) {$a_1'$};
\node[frozen] (A2b) at (1.15,-0.35){$a_2$};
\node[msg] (D2)  at (2.15,0.35) {$d$};
\node[msg] (R2)  at (3.1,0.35) {$v_R$};
\node[outc] (Y2) at (4.05,0.35) {$Y'$};
\draw[ed] (H2)--(A1b); \draw[ed] (H2)--(A2b);
\draw[ed] (A1b)--(D2); \draw[ed] (A2b)--(D2); \draw[ed] (D2)--(R2); \draw[ed] (R2)--(Y2);
\node[text=inkT,font=\tiny] at (1.72,-0.35) {\large$\blacklozenge$};
\node[text=inkA,align=center,text width=54mm] at (2,-2.35)
  {freeze $a_2$ at its logged value (teal lock), intervene $a_1$: $Y'{=}1$ still --- $a_1$ alone is not decisive. Changing the \emph{pair} gives $Y'{=}0$ $\Rightarrow$ minimal cause $\{a_1,a_2\}$, $\rho{=}\tfrac12$ each; $d$, $v_R$: $\rho{=}1$ as \emph{conduits}.};
\end{scope}

% ============ PANEL (c) : seeded replay to risk ============
\begin{scope}[xshift=128mm]
\node[pnl] at (1.6,2.05) {(c) seeded replay $\to$ risk};
\node[intv] (A3) at (0,0.35) {$a_1'$};
\node[msg] (R3a) at (1.4,1.15) {roll 1};
\node[msg] (R3b) at (1.4,0.35) {roll 2};
\node[msg] (R3c) at (1.4,-0.45){roll $K$};
\draw[ed] (A3)--(R3a);\draw[ed] (A3)--(R3b);\draw[ed] (A3)--(R3c);
\node[outc] (Y3a) at (2.9,1.15) {$Y{=}0$};
\node[outc] (Y3b) at (2.9,0.35) {$Y{=}0$};
\node[outc] (Y3c) at (2.9,-0.45){$Y{=}1$};
\draw[ed] (R3a)--(Y3a);\draw[ed] (R3b)--(Y3b);\draw[ed] (R3c)--(Y3c);
\draw[inkG] (0.2,-1.45) rectangle (3.4,-1.05);
\fill[inkA!55] (0.2,-1.45) rectangle (0.9,-1.05);
\node[text=inkG,align=center,text width=52mm] at (1.9,-2.35)
  {text steps re-run under logged seeds; \emph{physics} rolled from the $\phi$ checkpoint; $K$ rollouts give $\hat r=\Pr[Y{=}1\mid do(\cdot)]$ with a CI.};
\end{scope}

\end{tikzpicture}}
\caption{\textbf{The causal machinery on an intra-ensemble redundancy incident.} Inside the planning ensemble, planners $a_1$ and $a_2$ each produce the unsafe plan (either alone suffices); a single dispatcher $d$ emits the one command the robot executes. \textbf{(a)}~The \emph{production gate} keeps exactly the messages whose signed chain reaches the harm---and here the record is decisive: whether $a_2$ is a half-responsible joint cause or a preempted backup with $\rho=0$ is decided by \emph{which proposals $d$'s command cites}, a fact the enforced lineage records and no transcript reconstruction can recover. \textbf{(b)}~Grading: intervening on $a_1$ while freezing $a_2$ at its recorded value leaves the incident intact; only changing the pair flips it, so $\{a_1,a_2\}$ is the minimal cause and each carries $\rho=\tfrac12$---where single-site counterfactual scoring returns zero for both. The dispatcher and robot are but-for causes ($\rho{=}1$) flagged as \emph{conduits} by clean-input replay (Supplementary Note~7). \textbf{(c)}~Counterfactual outcomes are estimated by seeded re-execution plus physics rollouts from the $\phi$ checkpoints.}
\label{sfig:causal}
\end{figure}
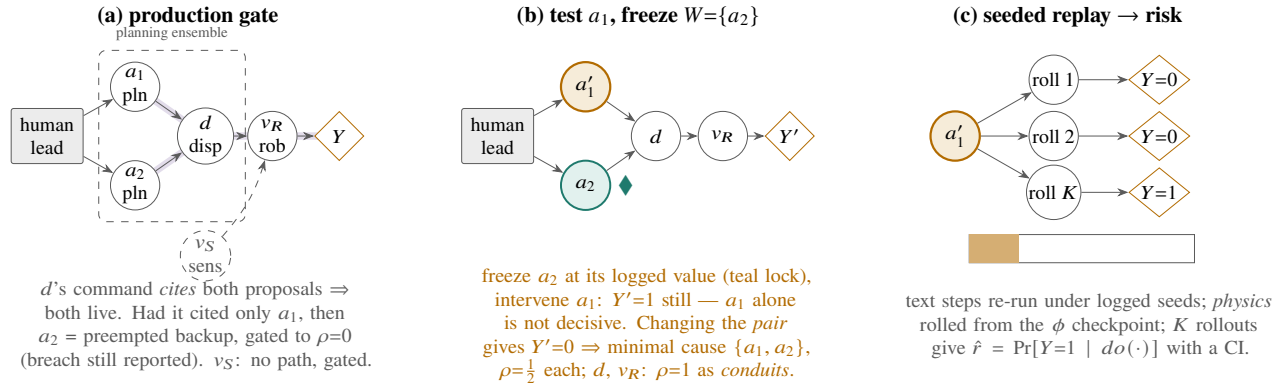

\begin{table}[t]
\centering
\caption{\textbf{Planted ground truth for the named families.} Grades from
the analytic oracle; flags and breach findings from the family
specification. Conduit-flagged grades are reported with the flag attached,
per the main text.}
\label{stab:gt}
\small
\begin{tabular}{@{}lllll@{}}
\toprule
Family & Principal & $\rho^{\mathrm{gt}}$ & Flags & Breach \\
\midrule
Redundancy & $a_1$, $a_2$ & $\tfrac12$, $\tfrac12$ & --- & --- \\
    & $d$, $r$ & $1$, $1$ & conduit & --- \\
Inert twin & $a_1$ & $1$ & --- & --- \\
    & $a_2$ & $0$ & exonerated & review skipped \\
Preemption & $a_3$ & $1$ & --- & --- \\
    & $a_1$ & $0$ & exonerated (gated) & hazardous issuance \\
Omission & $s$ (via $\bar H_s$) & $1$ & omission & halt duty \\
    & issuer of $c$ & $1$ & --- & --- \\
Delegation chain & $u,p_1,p_2,d$ & $1$ each & $p_2,d$: conduit & $p_1$: deviation \\
No-cause control & all & $0$ & exonerated & --- \\
Trivial control & author & $1$ & --- & (as drawn) \\
\bottomrule
\end{tabular}
\end{table}

\subsection{Grounding the incident library in public accident data}
\label{ssec:grounding}
The register's incident structures are grounded in a corpus of public
robot-accident records compiled for this paper (42 entries: OSHA IMIS
accident narratives under the keyword ``robot,'' 1984--2024, with 28 detail
records fetched individually; NIOSH and state FACE investigation reports; a
published analysis of 41 robot-related fatalities in United States CFOI data
1992--2017; and published accident-pattern taxonomies). The mapping from
observed causal shapes to register families, with per-shape counts from the
corpus, is: single-site command error (11) to preemption and record attack and the sampled
register's single-fault draws; omission of a duty-bearing safeguard,
lockout not applied, guarding defeated, supervisor absent (11), to omission and
the sampled register's omission axis; unexpected startup during maintenance
(7), the modal fatal pattern in the CFOI analysis, to the held-out
unexpected-startup family, in which every issued command is duty-compliant and the sole breach is
the omitted pre-entry halt; preemption and defeated protections (6) to the inert twin
and the lockout axis; mixed conjunctive failures (6) to the sampled
register's compositional draws; chain delegation (1) to the delegation chain. Two provenance
limits are stated rather than smoothed over. Overdetermination (the redundancy family)
is \emph{absent} from the accident corpus: every multi-failure incident in
it is conjunctive. Its provenance is the legal and causal-theory
literature on overdetermined harm, not accident data, and we keep it
because a graded calculus must handle the structure the doctrine treats as
hard. Chain delegation appears exactly once in the corpus; its prominence
in our register reflects the delegation depth of agent systems, not the
frequency observed in industrial records to date.

\section{Canonical form and granularity invariance}
\label{snote:canonical}

A principal should not be able to change its responsibility by re-chunking
its own output: splitting one command into three messages or merging three
into one. Two mechanisms make this so: the per-principal maximum of Eq.~(2)
of the main text, and a canonical contraction.

\begin{sdefinition}[Same-author conjunctive group]
\label{sdef:group}
Messages $m_1,\dots,m_k$ by the same author form a conjunctive group in
$\mathcal{M}_S$ if every mechanism depends on them only through the
conjunction $m_1\wedge\cdots\wedge m_k$ (formally: the composed outcome
function is invariant under permuting the group and under replacing the
group by a single variable equal to their conjunction). The
\emph{canonical contraction} replaces the group by one representative
variable $g$.
\end{sdefinition}

\begin{sproposition}[Granularity invariance]
\label{sprop:granularity}
On monotone slice models, canonical contraction leaves every principal's
grade unchanged: $\rho_p(\mathcal{M})=\rho_p(\mathcal{M}/g)$ for all $p$.
\end{sproposition}

\begin{proof}
By Lemma~\ref{lem:witness}, minimal causes are minimal transversals of the
satisfied prime implicants (Proposition~\ref{sprop:ness}), so it suffices
to show the contraction induces a size-preserving correspondence between
minimal transversals touching the group and those containing $g$.

\emph{No minimal cause contains two group members.} If
$m_1,m_2\in X^{\ast}$, note that setting $m_1\leftarrow 0$ already sets the
conjunction to $0$; since mechanisms see the group only through the
conjunction, the propagated evaluation of $X^{\ast}\setminus\{m_2\}$ equals
that of $X^{\ast}$, so $X^{\ast}$ was not minimal.

\emph{Projection.} Let $X^{\ast}$ be a minimal cause containing exactly one
member $m_i$. Replacing $m_i$ by $g$ yields a set of equal size that
falsifies $Y$ in $\mathcal{M}/g$ (setting $g\leftarrow0$ has the same
effect on every mechanism as setting $m_i\leftarrow0$), and it is minimal
there: a strictly smaller falsifying subset would lift (below) to a
strictly smaller falsifying subset in $\mathcal{M}$, contradicting
minimality. Causes disjoint from the group are unaffected.

\emph{Lift.} Let $X^{\ast}\ni g$ be a minimal cause in $\mathcal{M}/g$.
Replacing $g$ by any single member $m_i$ yields an equal-size set
falsifying $Y$ in $\mathcal{M}$ (same effect on all mechanisms), and it is
minimal by the symmetric argument.

The correspondence preserves sizes and authorship of the group (all members
share the author), so the sets over which each principal's maximum in
Eq.~(2) ranges have identical grade multisets, and $\rho_p$ is unchanged.
\qedhere
\end{proof}

\begin{sremark}
Even without contraction, Eq.~(2)'s maximum absorbs a principal's
\emph{own} chunking: if a size-$s$ minimal cause uses one of $p$'s chunks,
its grade $1/s$ is what $p$ receives no matter how many sibling chunks
exist. The contraction is a normalisation that additionally stabilises the
grades of \emph{other} principals whose causes interact with the group, and
it is applied before grading in all experiments, and the released test suite
checks the invariance on a re-chunked record.
\end{sremark}

\section{Record substrate and replay details}
\label{snote:substrate}

\paragraph{Storage tiers.}
The record is held in three tiers: a hot append-only log at the recorder; a
warm content-addressed store for payloads (only hashes are sealed,
so bulky or personal payloads can be encrypted and crypto-shredded for data
protection without breaking the seal); and cold sealed epochs whose Merkle
roots are published append-only. Two deployment profiles instantiate the
seal: Profile~A (single operator) publishes roots to a signed transparency
log with periodic external anchoring; Profile~B (multi-stakeholder) has a
keeper set (operator, vendors, insurer) co-sign each epoch, so no single
party can rewrite or fork history. Under either profile the record a reader
verifies is the one drawn in Figure~1 of the main text: hash-chained blocks,
each sealing an epoch of signed messages under a Merkle root.

\paragraph{Equivocation.}
Presenting different histories to different parties is defeated by the
single sealed sequence: receipts and epoch commitments bind every consumer
to the same root, and divergent roots are themselves cryptographic evidence
of recorder misbehaviour, surfacing in the suspect set.

\paragraph{Replay stacks and action-match.}
The pinned stack fixes model weights, decoding parameters, seeds, tool
versions, and the simulator build, making factual replay bit-exact and
interventional replay deterministic. The hosted stack replays against
components that cannot be pinned; certification then reports the confidence
interval of Proposition~1 together with an \emph{action-match} score: under
an order-preserving alignment of the replayed and recorded actuation
sequences, the fraction of aligned events agreeing in type, target, and
parameters within declared tolerances. A low action-match flags that the
replayed system is no longer the recorded system, independent of the risk
estimate.

\section{Reproducibility}
\label{snote:repro}

The reference implementation is a Python package organised as
\texttt{record/} (accountable messages, verification, Merkle sealing),
\texttt{engine/} (gate, grade, certify), \texttt{verdict/} (analytic
responsibility, intervals, canonical form), \texttt{facility/} (incident
families, generator, analytic oracle), \texttt{corpus/} (live pipeline,
replay oracle), \texttt{baselines/}, and \texttt{metrics/}, with one per experiment identifier under \texttt{scripts/}. Each
runner writes a JSON artifact under \texttt{results/} carrying the
configuration, per-incident outputs, and summary statistics, and every run
sweeps generation seeds $k\in\{0,1,2\}$ where the register is stochastic.
All numbers, figures, and tables in the main text are regenerated from those
artifacts by scripts included in the release, and a verification script
re-checks every reported quantity against its artifact before submission. The simulation layer is CPU-deterministic; pinned-stack replay
of open-weight planners uses seeded deterministic serving. Experiments ran on NVIDIA A100-SXM4-80GB GPUs
under Python 3.11, NumPy 2.4, and SciPy 1.17, with open-weight models served
by vLLM 0.14.1; the archived release pins exact package versions and the
container digest.

% CANONICAL NOTE NUMBERING (update main-text plain references if sections move):
% 1 Notation, 2 Proofs, 3 Adversarial guarantees, 4 Framing-as-implemented, 5 Scenario,
% 6 Registered predictions, 7 Canonical form, 8 Record substrate,
% 9 Reproducibility, 10 Disclosures, 11 Extended related work
\section{Experimental deviations and validation}
\label{snote:disclosures}
Three protocol changes affect interpretation of the reported experiments. The Llama-3.1-8B cross-family evaluation was extended from 13 to 25 scoreable incidents under unchanged arms, metrics, and exclusion rules; the original and extended estimates are retained in the experiment log. The planned Llama-3.3-70B comparison was replaced by the matched-scale Qwen2.5-7B versus Llama-3.1-8B comparison because the former could not be served under the available cluster capacity. The external evaluation used Who\&When rather than Who\&When Pro because the latter’s data and code were unavailable at evaluation time. Independent cross-implementation checks were performed before final evaluation, and all reported results were generated or re-verified using the finalized implementation.

% audita-si-extended-related.tex — SI: extended related-work survey.
% The closest lines are in the main-body Related Work; this note covers the rest.
\section{Extended related work}
\label{snote:extrelated}

The main-body Related Work covers the lines closest to \audita{}: failure
attribution in multi-agent systems, actual causality and graded
responsibility, and the manipulation-of-attribution literature that flanks our
blame-shift certificate. This note provides the extended survey across the
remaining lines the project engaged, organised by the role each plays relative
to \audita{}.

\subsection*{Strategic and adversarial manipulation of attribution}
A small adjacent literature studies attribution when the attributed parties
are strategic, and none of it, to our knowledge, addresses the framing
problem our guarantees target. In cooperative sequential decision making,
game-theoretic blame attributions provably misalign incentives
(non-performance-incentivising Shapley, over-blaming Banzhaf), motivating
attribution methods with better structural properties under
\emph{uncertainty}, but the agents there do not attack the attribution of
\emph{others}\cite{triantafyllou2021}. Recent work on retrospective
counterfactual responsibility in concurrent stochastic games lets agents
trade their \emph{own} expected responsibility against reward in
equilibrium, forward-looking strategy synthesis, with no adversarial
manufacture of a victim's responsibility and no evidence
layer\cite{mu2026responsibility}. In the explanation literature,
adversarially crafted models can fool perturbation-based feature
attributions such as LIME and SHAP\cite{slack2020fooling}; the object
manipulated there is a model explanation, not a principal's graded
responsibility over a certified record. Accountability systems in the
PeerReview tradition guarantee that correct nodes are never exposed, but
detect \emph{protocol deviation} against a deterministic reference
implementation\cite{haeberlen2007}; our Attack~I is protocol-compliant
conduct, invisible to deviation detection by construction. Our blame-shift
accountability theorem differs from all four lines in both question and
machinery: it asks whether a coalition can raise a \emph{victim's}
responsibility, answers by applying the Halpern--Pearl test at the meta
level (deviations as causes of the victim's pivotality), and returns a
replay-checkable certificate that grades the framers. On the barrier side,
concurrent work establishes statistical limits of ground-truth-free
auditing, worst-case calibration error of any label-free estimator in the
rare-error regime\cite{wang2026verificationtax}; our completeness barrier
is the per-instance, adversarially-relevant counterpart: a logical
soundness/completeness contradiction for content-determined breach
standards, tied to negligence semantics rather than to estimator
calibration.

\subsection*{Causation in law, defaults, and omissions}
The legal doctrine \audita{} answers to is the NESS test, a cause as a
necessary element of a sufficient set of actual conditions\cite{wright1985}
, and the broader treatment of causation in the law by Hart and
Honor{\'e}\cite{harthonore1985}, from which the verdict's
duty--breach--causation--harm structure is drawn. The philosophical
literature distinguishes dependence from production as two concepts of
causation\cite{hall2004}, precisely the distinction our gate-and-grade test
operationalises; graded causation with defaults and normality gives
omissions their standing as causes\cite{halpernhitchcock2015}, which our
duty-indexed absence variables implement mechanically. Empirically, human
causal judgments track counterfactual simulation\cite{gerstenberg2021},
and human responsibility intuitions are sensitive to factors, such as how
critical or salient a contributor appears, that a defensible audit must
not inherit; our volume-confound experiment tests the machine
analogue of exactly this failure.

\subsection*{The responsibility gap and AI accountability}
That learning systems open a gap between harm and accountable agent is the
seminal observation of Matthias\cite{matthias2004}, refined into four
distinct gaps, culpability, moral accountability, public accountability,
and active responsibility\cite{santonidesio2021}, and rooted in the older
problem of many hands\cite{nissenbaum1996}. \audita{} targets the
culpability and public-accountability gaps by manufacturing the evidentiary
object they presuppose. Governance work on agent visibility proposes
identifiers, real-time monitoring, and activity logs\cite{chan2024};
authenticated-delegation frameworks issue scoped, auditable authority from
humans to agents\cite{south2025}; and on-chain identity registries are
commoditising agent identity at scale\cite{erc8004}, \audita{} consumes
these primitives (its delegation certificates and duty clauses are exactly
such tokens) and adds what they do not attempt: post-hoc causal
attribution over the recorded conduct. System- and process-level AI
auditing assesses documentation, monitoring, and governance of deployed
systems\cite{auditmai2024,fernsel2024auditability}; regulation demands
event recording and traceability\cite{euaiact2024} within safety regimes
that presume incidents can be reconstructed\cite{iso10218,isots15066}.
\audita{} audits the event rather than the system, and proposes what the
mandated records must \emph{be} for those obligations to purchase
accountability.

\subsection*{Accountability systems and tamper-evident records}
The distributed-systems ancestry of the record is direct. PeerReview
established that nodes signing all messages into tamper-evident logs,
audited by witnesses, yields two guarantees, detected faults are
irrefutably linked to a faulty node, and correct nodes can always defend
against false accusation\cite{haeberlen2007}; Theorem~1 is our
generalisation of the second guarantee from protocol deviation to graded
semantic responsibility. Accountable Virtual Machines extended the recipe
to record-replay-blame over full executions\cite{haeberlen2010}, and its
stated limits are exactly our deltas: it requires a deterministic reference
implementation and cannot fault faithful execution of flawed
software, language-model principals have no reference implementation, are
nondeterministic, and act on the physical world. CSAR made accountability
survive randomness by logging application inputs and random
choices\cite{backes2009}, the ancestor of our seed-and-model-hash fields;
formal definitions of accountability via judge-rendered verdicts, and
their relation to verifiability, come from K{\"u}sters, Truderung and
Vogt\cite{kusters2010}. The data structures beneath the seal are the
tamper-evident history trees of Crosby and Wallach\cite{crosby2009}, fork
consistency against equivocating servers from SUNDR\cite{li2004sundr},
forward-secure audit logs\cite{schneierkelsey1999}, Merkle
commitments\cite{merkle1988}, and transparency logging at Internet
scale\cite{laurie2014}, under classical Byzantine and network-adversary
models\cite{lamport1982,dolev1983}. Practitioner standardisation is
beginning: an IETF draft specifies hash-chained, optionally signed audit
records for agents with an explicit mapping to the AI Act's logging
article\cite{ietf2026audittrail}, a single-writer schema without lineage,
receipts, or physical binding, whose spirit our record extends. The
nearest new relative is the verifiable-transcript line, which applies
SUNDR-style signed digests and fork detection to LLM conversation
transcripts under an untrusted-provider threat model\cite{xing2026vct};
it certifies a two-party transcript, whereas \audita{} certifies a
multi-principal command graph with authored lineage, actuation binding,
and attribution on top.

\subsection*{Agent observability tooling}
Industry telemetry for agents is consolidating around vendor-neutral
semantic conventions for agent, tool, and model spans\cite{otelgenai}, on
the pedigree of large-scale distributed tracing\cite{sigelman2010};
academic work taxonomises what agent operations should
trace\cite{dong2024agentops} and how agentic systems can be observed and
optimised in operation\cite{moshkovich2025}. The gap this cluster leaves
is the one \audita{} fills: telemetry captures what happened without
identity, integrity, or assessment, no field binds a span to a signing
principal, nothing proves the trace complete, and no semantics turns
spans into responsibility. Our message format is deliberately close to
these conventions so that the three fields standard logs lack (authored
citations, committed effects, pinned seeds) read as an extension rather
than a replacement.

\subsection*{Threats to agent and robot collectives}
The attack literature supplies both motivation and the record-corruption manipulations used in our evaluation.
Jailbreaking attacks produce harmful physical actions on deployed
LLM-controlled robots\cite{robey2025}; in LLM-coordinated multi-robot
teams, compromising a single entry robot propagates malicious intent
through peer communication, with obedience reaching 1.00 and full-team
compromise in as few as three rounds in the strongest reported
cases\cite{huang2026compromise}, the staged version of the incidents
\audita{} investigates. Surveys of agent communication protocols document
tool poisoning, injection, and cross-boundary provenance
gaps\cite{li2025commsec}. Prevention-oriented security frameworks harden
the channel: the Aegis Protocol combines decentralised identifiers,
post-quantum signatures, and zero-knowledge policy compliance under an
extended Dolev--Yao adversary\cite{adapala2025}, and BlockA2A anchors
agent-to-agent interoperability in verifiable
identity\cite{zou2025blocka2a}. These are complements, not competitors:
they aim to prevent the incident, \audita{} to adjudicate it, and a
hardened channel makes the certified record's assumptions easier to
discharge.

\subsection*{Safety science and incident investigation}
The incumbent methodology for accident analysis is systems-theoretic:
STAMP models accidents as control-structure failures and CAST analyses
them deliberately blame-free, asking why and how rather than
who\cite{leveson2011}. \audita{} is designed to feed, not replace, that
practice: the same certified record supports a blame-free learning
analysis and the evidentiary who-and-how-much account that liability,
insurance, and regulation additionally demand. The ethical black box
argued that robots need flight recorders specified for accident
investigation\cite{winfield2017}; our record is its multi-principal,
agentic-era successor, and it attributes as well as logs. Adjacent
blockchain-robotics work uses ledgers for swarm coordination and
Byzantine defence, token economies that neutralise harmful
robots\cite{strobel2023}, reviewed for mobile multi-robot systems in the
Nature portfolio\cite{dorigo2024}; that line secures coordination at run
time, while \audita{} uses ledger machinery for post-incident causal
attribution.

\subsection*{Neighbouring methodological fields}
Three method families adjoin ours and mark its boundary. Root-cause
analysis for microservices computes type-level causal effects over
inferred dependency graphs, increasingly with partial identification
under latent confounding and robustness to graph misspecification; our
problem is token-level actual causation over a graph that is signed by
construction rather than inferred, with uncertainty that is adversarial
rather than statistical, and with graded responsibility, exoneration,
and an innocence guarantee that root-cause outputs do not carry.
Provenance-based forensics captures whole-system operating-system
provenance and reconstructs attack stories over
it\cite{pasquier2017}, the mature syscall-level analogue of our
problem, without message semantics, duties, or embodiment. Causal
statistical fault localisation pioneered causal inference over program
dependence graphs\cite{baah2010}, and its slicing-to-the-cause-effect-chain
pattern prefigures our production gate, ported here to signed
cross-principal command graphs with physical outcomes. Finally,
counterfactual credit assignment in cooperative multi-agent reinforcement
learning\cite{foerster2018,shapley1953} shares the counterfactual core
but serves training-time reward shaping; post-hoc evidentiary
attribution differs in object (a concrete incident), in adversary (the
record itself is attacked), and in consumer (parties with legal stakes),
which is why the game-theoretic attributions it favours misbehave as
responsibility measures\cite{triantafyllou2021}.

\subsection*{The deployment substrate}
The systems whose incidents \audita{} is built to audit are documented in
the main text: language-grounded robot planners and vision-language-action
models\cite{ahn2022saycan,driess2023palme,brohan2023rt2}, open-ended and
communicative agents\cite{wang2023voyager,li2023camel}, and orchestration
frameworks\cite{wu2023autogen,hong2023metagpt}, evaluated on
field-standard task suites\cite{mialon2023gaia,yoran2024assistantbench,
jimenez2024} whose shapes our facility simulation abstracts. Across all
these clusters, the pattern is the one the paper argues: each ingredient of
\audita{} has deep roots in its home field, and what is new is the
composition, certified evidence, production-gated graded causation, and
replay verification bound into one pipeline whose output is designed to
survive a hostile audience.

\end{document}